\documentclass{article}
\usepackage{amsmath,amsthm,amssymb}
\usepackage{fullpage}

\usepackage{algorithm}
\usepackage{algpseudocode}

\usepackage{xcolor}
\definecolor{DarkGreen}{rgb}{0.1,0.5,0.1}

\usepackage[utf8]{inputenc} 
\usepackage[T1]{fontenc}    
\usepackage{hyperref}       
\usepackage{url}            
\usepackage{booktabs}       
\usepackage{amsfonts}       
\usepackage{nicefrac}       
\usepackage{microtype}      
\usepackage{lipsum}
\usepackage{fancyhdr}       
\usepackage{graphicx}       
\graphicspath{{media/}}     

\usepackage{thm-restate}
\usepackage{bm}

\DeclareMathOperator*{\argmin}{arg\,min}
\usepackage{dsfont} 

\newcommand{\D}{\mathcal{D}}
\newcommand{\cL}{\mathcal{L}}

\newcommand{\eps}{\epsilon}
\newcommand{\f}{\bm{f}}
\newcommand{\E}{\mathbb{E}}
\newcommand{\R}{\mathbb{R}}

\usepackage[english]{babel}

\newtheorem{definition}{Definition}[section]
\newtheorem{theorem}{Theorem}[section]
\newtheorem{lemma}[theorem]{Lemma}
\newtheorem{remark}[theorem]{Remark}
\newtheorem{corollary}[theorem]{Corollary}
\newtheorem{assumption}[theorem]{Assumption}

\begin{document}
  
\title{Differentially Private Synthetic Control}

 \renewcommand{\thefootnote}{\fnsymbol{footnote}}

\author{Saeyoung Rho\footnotemark[1] \and 
Rachel Cummings\footnotemark[1] \and 
Vishal Misra\footnotemark[1] }

\footnotetext[1]{Columbia University. Emails: \texttt{\{s.rho, rac2239, vishal.misra\}@columbia.edu}}

 \renewcommand{\thefootnote}{\arabic{footnote}}

\maketitle

\begin{abstract}
Synthetic control is a causal inference tool used to estimate the treatment effects of an intervention by creating synthetic counterfactual data.
This approach combines measurements from other similar observations (i.e., \emph{donor pool}) to predict a counterfactual time series of interest (i.e., \emph{target unit}) by analyzing the relationship between the target and the donor pool before the intervention.
As synthetic control tools are increasingly applied to sensitive or proprietary data, formal privacy protections are often required.
In this work, we provide the first algorithms for differentially private synthetic control with explicit error bounds. Our approach builds upon tools from non-private synthetic control and differentially private empirical risk minimization.
We provide upper and lower bounds on the sensitivity of the synthetic control query and provide explicit error bounds on the accuracy of our private synthetic control algorithms. We show that our algorithms produce accurate predictions for the target unit, and that the cost of privacy is small.
 Finally, we empirically evaluate the performance of our algorithm, and show favorable performance in a variety of parameter regimes, as well as providing guidance to practitioners for hyperparameter tuning.

\end{abstract}

\section{Introduction}\label{s.intro}

The fundamental problem of causal inference \cite{rubin1974} is that for an individual unit, we can only observe one of the relevant outcomes -- with a particular treatment or without. To estimate the (causal) effect of a treatment, one has to produce a counterfactual of the control arm, which is typically done at a population- and distributional-level via randomized control trials (RCTs) and A/B testing, yielding average treatment effects. However, controlled trials are often impossible to implement, and only observational data are available. 
Synthetic control is a powerful causal inference tool to estimate the treatment effect of interventions using only observational data. It has been used both at an aggregate population level (e.g., countries/cities/cohorts of patients etc.), as well as at an individual unit level \cite{mrsc,R4.5}, and has been called ``arguably the most important innovation in the policy evaluation literature in the last 15 years'' \cite{athey2017state}. 

Recently, synthetic control has increasingly been used in clinical trials where running a randomized control trial presents logistical challenges (e.g., rare diseases), or ethical issues (e.g., oncology trials enrolling patients for placebos with life threatening diseases) \cite{synthetic_control_arms}. Synthetic control has been successfully used to achieve regulatory approval for new medical treatments for lung cancer \cite{petrone2018roche} and rare forms of leukemia \cite{gokbuget2016blinatumomab}, where RCTs would otherwise have been impossible.
Since these synthetic control analyses are deployed in real-world medical applications, preserving privacy of sensitive patient data is paramount.
    
Differential privacy \cite{DMNS06} has emerged as the de facto gold-standard in privacy-preserving data analysis. It is a mathematically rigorous parameterized privacy notion, which bounds the maximum amount that can be learned about any data donor based on analysis of her data. Differentially private algorithms have been designed for a wide variety of optimization, learning, and data-driven decision-making tasks, and have been deployed in practice by several major technology companies and government agencies. Despite the growing maturity of the differential privacy toolkit, the pressing need for a private synthetic control solution has thus far gone unaddressed.

Our work provides the first differentially private algorithms for synthetic control with provable accuracy and privacy guarantees.

\subsection{Our Contributions}\label{s.contributions}

Our main contributions are the first algorithms for differentially private synthetic control (Algorithm \ref{alg.output} and \ref{alg.obj}). These two algorithms naturally extend existing non-private techniques for synthetic control by first privately estimating the regression coefficients $\bm{\hat{f}}$ that relate a (target) pre-intervention observation of interest $\bm{y_{pre}}$ to other similar (donor) observations $X_{pre}$. This is done using output perturbation and objective perturbation techniques for differentially private empirical risk minimization (ERM) from \cite{CMS11, KST12}. The algorithm then combines the private regression coefficients $\bm{\hat{f}}$ with privatized post-intervention donor observations $\tilde{X}_{post}$ (also via output perturbation) to predict the post-intervention target outcome $\bm{\hat{y}_{post}} = \tilde{X}_{post}^{\top}\bm{\hat{f}}$.

Our main results are privacy and accuracy guarantees for each algorithm. For privacy (Theorems \ref{thm.priv.out} and \ref{thm.priv.obj}), although our algorithmic techniques rely on existing approaches, prior results on privacy do not apply in our setting. DP methods add noise that scales with the \emph{sensitivity} of the function being computed, which is defined as the maximum change in the function's output that can be caused by changing a single donor's data. However, synthetic control performs a regression in a vertical way, treating each time point, rather than each donor's data, as one sample -- thus, the \emph{transposed} setting changes the definition of neighboring databases, completely altering the impact of a single donor's data. The majority of our privacy analysis is devoted to computing sensitivity of this new method.

We also provide accuracy guarantees for each algorithm (Theorems \ref{thm.acc.out} and \ref{thm.acc.obj}), bounding the root mean squared error (RMSE) of the algorithm's output compared to the post-intervention target signal. Our bounds are comparable to those for non-private synthetic control (e.g., \cite{rsc}), and in Section \ref{s.privacycost}, we explicitly show that the cost of privacy in synthetic control is $O(1/\eps)$. That is, the RMSE of our algorithm relative to a non-private version is only greater by a factor of $O(1/\eps)$, which is unavoidable in most analysis tasks. To better interpret our bounds in terms of natural problem parameters such as number of samples and length of observations, we also provide Corollaries \ref{cor.accuracy} and \ref{cor.accuracy.obj}, which give explicit closed-form upper bounds on the RMSE under mild assumptions on the underlying data distribution.

\subsection{Related Work}\label{s.relwork}

\paragraph{Synthetic Control.}
Synthetic control (SC) was originally proposed to evaluate the effects of intervention by creating synthetic counterfactual data. Its first application was measuring the economic impact of the 1960s terrorist conflict in Basque Country, Spain by combining GDP data from other Spanish regions prior to the conflict to construct a \textit{synthetic} GDP dataset for Basque Country in the counterfactual world without the conflict \cite{abadie2003economic}.
Synthetic control has since been applied to a wide array of topics such as estimating the effect of California's tobacco control program \cite{abadie2010synthetic}, estimating the effect of the 1990 German reunification on per capita GDP in West Germany \cite{abadie2015comparative}, evaluating health policies \cite{kreif2016examination}, forecasting weekly sales at Walmart stores \cite{mrsc}, and predicting cricket score trajectories \cite{mrsc}.

The core algorithm of synthetic control lies in finding a relationship between the target time series (e.g., GDP of Basque Country) and the donor pool (e.g., GDP of other Spanish regions). The original method of \cite{abadie2003economic} used linear regression with a simplex constraint on the weights: the regression coefficients should be non-negative and sum to one.
Since its first introduction, the synthetic control literature has evolved to include a richer set of techniques, including tools to deal with multiple treated units \cite{abadie2021penalized, dube2015pooling}, to correct bias \cite{abadie2021penalized,ben2021augmented}, to use Lasso and Ridge regression instead of linear regression with simplex constraints \cite{doudchenko2016balancing, rsc}, and to incorporate matrix completion techniques \cite{athey2021matrix, rsc, mrsc}. See \cite{abadie2021using} for a detailed survey of these techniques.

The most relevant extension for our work is \emph{robust synthetic control} (RSC) \cite{rsc}, which comprises of two steps: first de-noising the data via hard singular value thresholding (HSVT), and then learning and projecting via Ridge regression. It assumes a latent variable model and applies HSVT before running the regression, which reduces the rank of the data and makes synthetic control more robust to missing and noisy data. RSC also relaxes the simplex constraints on the regression coefficients and applies unconstrained Ridge regression. Because of the de-noising step, robust synthetic control can be viewed as an instantiation of principal component regression (PCR), and the possibility of differentially private PCR has been briefly discussed 
\cite{agarwal2021robustness}. However, no formal algorithm or analysis has been put forth. We are the first to design and analyze differentially private algorithms for synthetic control.

\paragraph {Differentially Private Empirical Risk Minimization.} \cite{CMS11} first proposed methods for differentially private empirical risk minimization (ERM) for supervised regression and classification.
Our first algorithm uses the \emph{output perturbation} method from \cite{CMS11}, which first computes coefficients to minimize the loss function between data features and labels, and then perturbs the coefficients using a high-dimensional variant of the Laplace Mechanism from \cite{DMNS06}. Our second algorithm uses the \emph{objective perturbation} method  \cite{CMS11,KST12}, which adds noise directly to the loss function and then exactly optimizes the noisy loss.  This method tends to provide better theoretical accuracy guarantees, but requires the loss function to satisfy additional structural properties. These methods were later extended by \cite{bassily2014private} to include \emph{gradient perturbation} in stochastic gradient descent, which uses a noisy version of randomly sampled points' contribution to the gradient at each update. This technique provides tighter error bounds, assuming Lipschitz convex loss and bounded optimization domain. \cite{wang2017differentially} followed up with a faster gradient perturbation algorithm that provided a tighter upper bound on error and lower gradient complexity.

Although the framework of \cite{CMS11} is more general, the analysis and applications focused only on methods for binary classification. The analysis was later extended to include ridge regression in \cite{truthful}, which we use in our algorithms. Our algorithms for differentially private synthetic control apply differentially private ERM methods to a ridge regression loss function. However, synthetic control applies regression in the transposed dimension of the data (i.e., along columns rather than rows of the database), while privacy protections are still required along the rows, which requires novel analysis to ensure differential privacy and accuracy.

\section{Model and Preliminaries}\label{s.prelims}

In this section, we first present our synthetic control model (Section \ref{s.model}), and then provide relevant background on synthetic control (Section \ref{s.synthcont}) and differential privacy (Section \ref{s.dp}).

\subsection{Our Model}\label{s.model}

Our model follows the synthetic control framework illustrated in Figure \ref{fig.synthcon}. We consider a database $X \in \mathbb{R}^{n \times T}$, also called the \emph{donor pool}.  The donor pool $X$ consists of $n$ time series, each observed at times $t=1,\ldots,T$. We denote the column vectors of $X$ as $\bm{{x_1}}, \cdots \bm{{x_{T}}} \in \mathbb{R}^{n}$, where each $\bm{x_t}$ contains observations from all donor time series at time $t$. We assume an intervention occurred at a known time $T_0+1 < T$.  The first $T_0$ columns of $X$ are collectively referred to as $X_{pre}$, and the remaining $T-T_0$ columns from data after the intervention are collectively denoted $X_{post}$, respectively corresponding to the pre- and post-intervention donor data. We are also given a \emph{target unit} $\bm{y} \in \mathbb{R}^T$, which can be divided as $\bm{y_{pre}}=(y_1, \ldots, y_{T_0})$ and $\bm{y_{post}} = (y_{T_0+1},\ldots, y_T)$. 

\begin{figure}[h]
\centering
\includegraphics[width=0.5\textwidth]{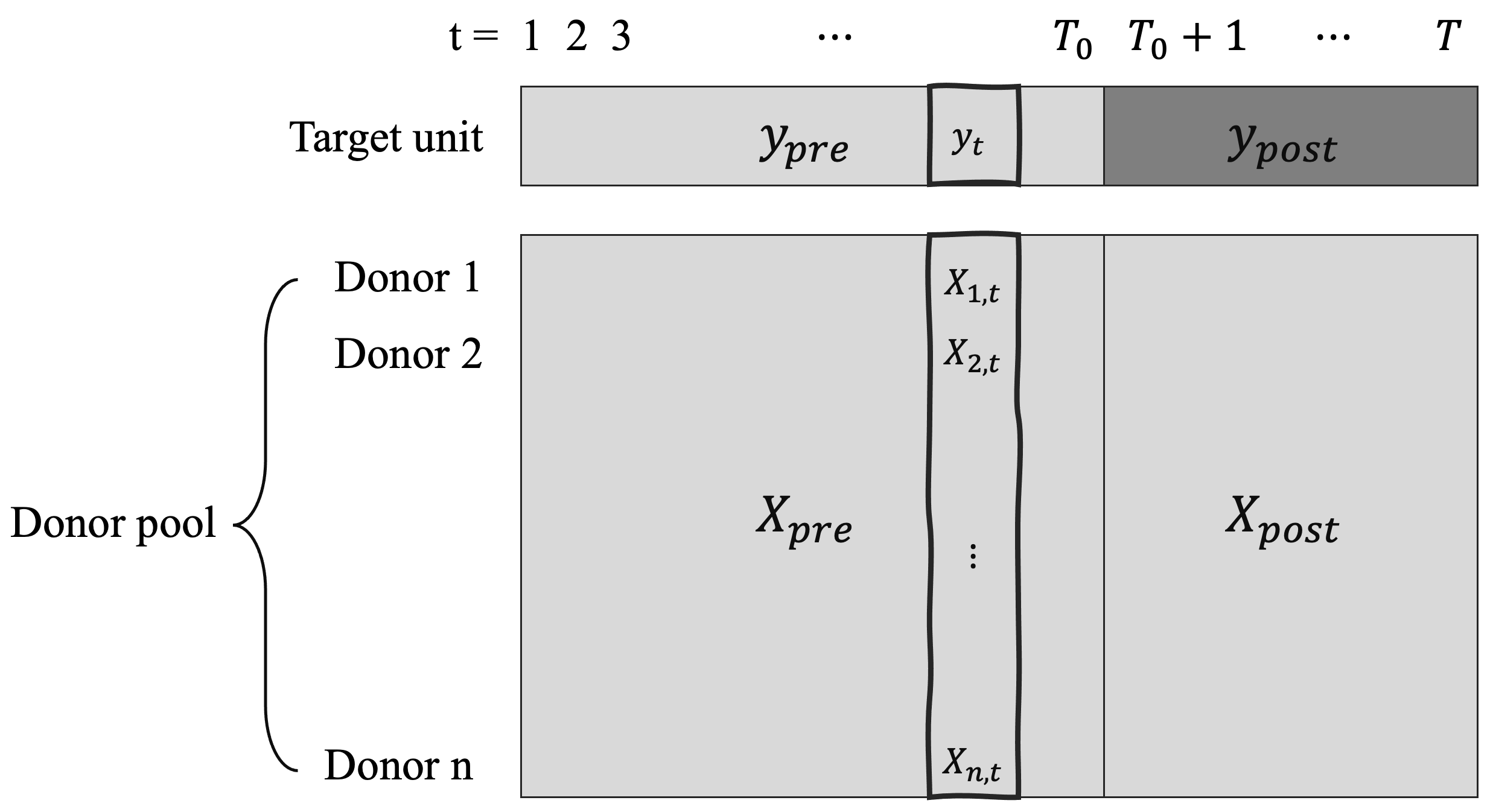}
\caption{General data structure for synthetic control. The donor pool ($X$) and the target unit ($\bm{y}$) are divided into pre- and post-intervention periods. Synthetic control first performs a \emph{vertical} regression using the pre-intervention column vectors $\bm{x_t}$ as features for the label $y_t$ for $t\in[T_0]$ to estimate regression coefficients $\hat{\f}$, and then uses this to project the post-intervention column vectors and predict $\bm{\hat{y}_{post}}$.
}\label{fig.synthcon}
\end{figure}

The underlying assumption is that time series in the donor pool that behave similarly to $\bm{y}$ before the intervention will remain similar after the intervention.
In this paper, we use the latent variable model of \cite{rsc} for the underlying distribution of the data. Our donor data and target data are noisy versions of the true signal (denoted $M$ and $\bm{m}$ respectively), and can be written as follows:
\begin{equation}\label{eq.noise}
X = M + Z, \qquad \bm{y} = \bm{m} + \bm{z},
\end{equation}
where $Z \in \mathbb{R}^{n \times T}$ is a noise matrix where each element is sampled i.i.d.~from some distribution with zero-mean, $\sigma^2$-variance, and support $[-s,s]$, and $\bm{z} \in \mathbb{R}^{T}$ is a noise vector with elements sampled from the same distribution.

The signals $M$ and $\bm{m}$ can be expressed in terms of a latent function $g$:
\[
M_{i,t} = g(\theta_i, \rho_t) \; \forall i \in [n], t \in [T], \qquad
m_{t} = g(\theta_0, \rho_t) \; , \forall t \in [T],
\]
where $\theta_i$ and $\rho_t$ are latent feature vectors capturing unit $i$'s and time $t$'s intrinsic characteristics, respectively. We note that if the intervention is effective, one would expect to see a change in $\rho_t$ before and after $T_0$. We make no assumptions on the latent function $g$, except in Sections \ref{s.accdist} and \ref{s.objclosedform}, where we assume $M$ is low rank, i.e., $rank(M) = k$ for some $k \ll \min\{n,T\}$.

Finally, we assume a linear relationship between the features of $M_{i,t}$ and the label $m_t$ at all times $t\in[T_0]$; that is, there exists an $\f \in \mathbb{R}^n$ such that,
\begin{equation}\label{eq.model}
 m_t = \sum_{i=1}^n M_{i,t} f_i, \quad \text{for all } t \in [T_0]. \end{equation}
We assume that all entries of $X$, $M$, $\bm{y}$, and $\bm{m}$ lie in a bounded range, which we rescale to $[-1,1]$ WLOG, and that $\f$ has $\ell_1$-norm bounded by 1, as is standard in the synthetic control literature \cite{abadie2003economic}. Formally, we assume:
\begin{equation}
\label{eq.assump}
     |x_{i,t}| \leq 1, |M_{i,t}| \leq 1 \; \forall t \in[T], \; i\in [n], \quad |y_t|\leq 1, |m_t|\leq 1 \; \forall t\in [T], \quad and \quad ||\f||_1 = \sum_{k=1}^{n} |f_k| \leq 1. 
\end{equation}

\subsection{Synthetic Control}\label{s.synthcont}

The goal of synthetic control is to predict $\bm{y_{post}}$ given $X$ and $\bm{y_{pre}}$.  The general approach, outlined in Algorithm \ref{alg.synth}, is to first use the pre-intervention data $\mathcal{D}_1 := (X_{pre},\bm{y_{pre}})$ to learn an estimate $\bm{\hat{f}}$ of the true coefficient vector $\f$. For each $t\in[T_0]$, the column vector $\bm{x_t} = (X_{1,t}, \cdots, X_{n,t})^{\top}$ is treated as a feature vector for label $y_t$. This setup distinguishes synthetic control from the classic regression setting, as the regression is performed vertically rather than horizontally. The estimate $\bm{\hat{f}}$ is then used along with the post-intervention donor data to predict the counterfactual outcome of the target: $\bm{\hat{y}_{post}} = X^{\top}_{post} \bm{\hat{f}}$, where $\hat{y}_{t} = \bm{x_t}^{\top}\bm{\hat{f}}$  $\forall t \in \{T_0+1, \cdots, T\}$.

\begin{algorithm}
\caption{Synthetic control framework ($X, \bm{y_{pre}}, T_0,J$)}\label{alg.synth}
\begin{algorithmic}
\State Divide $X$ into pre- and post-intervention observations
\begin{equation*}
    X = (X_{pre}, X_{post}) = \left(
    \left( \begin{matrix}
x_{1,1} & \cdots & x_{1,T_0}\\
\vdots & \ddots & \vdots \\
x_{n,1} & \cdots & x_{n,T_0}\\
\end{matrix} \right)
\left( \begin{matrix}
x_{1,T_0+1} & \cdots & x_{1,T}\\
\vdots & \ddots & \vdots \\
x_{n,T_0+1} & \cdots & x_{n,T}\\
\end{matrix} \right)
\right)
\end{equation*}
\State \textbf{Step 1: Learn regression coefficients}
\[\bm{\hat{f}} = \argmin_{\f\in \mathbb{R}^{n}} J(\f; X_{pre}, \bm{y_{pre}})\]
\State \textbf{Step 2: Predict $\bm{y_{post}}$ via projection}
\State Output $\bm{\hat{y}_{post}} = {X}_{post}^{\top} \hat{\f} \in \mathbb{R}^{T-T_0}$
\end{algorithmic}
\end{algorithm}

When synthetic control is used for evaluating treatment effects, it is assumed that the target received a different treatment from the donor pool, and the goal is to predict the counterfactual outcome under the alternative treatment. In this case, the treatment effect is evaluated as the difference between the observed $\bm{y_{post}}$ and the counterfactual prediction $\bm{\hat{y}_{post}}$. Even if no intervention occurred at time $T_0+1$ (or if the target received the same treatment as the donor pool), then synthetic control can also be used to predict future observations of the target time series. In that case, accuracy can be measured as the difference between the actual observation $\bm{y_{post}}$ and the predicted $\bm{\hat{y}_{post}}$.
While the former task is the more common use-case for synthetic control, we focus our attentions in this work on the latter, in order to cleanly evaluate accuracy of our algorithms' predictions without confounding treatment effects.

The original synthetic control work of \cite{abadie2003economic} learned regression coefficients using ordinary linear regression with a simplex constraint on $\bm{\hat{f}}$, i.e., $\widehat{f}_i\geq 0 \; \forall i\in[n]$ and $\sum_{i\in[n]} \widehat{f}_i =1$. Later works such as \cite{rsc, mrsc, doudchenko2016balancing, ben2021augmented} used penalties such as LASSO, Ridge, and elastic net regularizers.

In this work, we use \emph{Ridge regression} (with empirical loss $\cL(\f; X, y)=\frac{1}{T_0}||y -X^\top \f||_2^2$ and an $\ell_2$ regularizer $r(\f) = \frac{\lambda}{2T_0} || \f ||_2^2$) to estimate $\hat{\f}$, which corresponds to the following regularized quadratic loss function:
\begin{equation}\label{eq.scloss}
\begin{split}
    J(\f;\D) & = \cL(\f; X_{pre}, \bm{y_{pre}}) + r(\f) = \frac{1}{T_0} \sum_{t=1}^{T_0} (y_t - \bm{{x_t}}^\top \f )^2 + \frac{\lambda}{2T_0} || \f ||_2^2.
\end{split}
\end{equation}

\emph{Robust synthetic control} (RSC) extends this framework to include a data pre-processing step to denoise $X$ using hard singular value thresholding (HSVT) \cite{rsc}. RSC first performs singular value decomposition on $X$, and truncates all singular values below a given threshold to be 0. This serves to reduce the rank of $X$ before the learning and prediction steps are performed.

\subsection{Differential Privacy}\label{s.dp}

Differential privacy \cite{DMNS06} ensures that changing a single user's data will have only a bounded effect on the outcome of an algorithm. Specifically, it ensures that the distribution of an algorithm's output will be similar under two \emph{neighboring databases} that differ only in a single data record.
In the synthetic control setting, where the analysis goal is to predict the post-intervention target unit $\bm{y_{post}}$ using the donor pool $X$ and its relationship to $\bm{y_{pre}}$, we aim to predict privacy of data records in $X$ but not $\bm{y_{pre}}$, since the target will know their own pre-intervention data.
Note that this is similar to the notion of \emph{joint differential privacy} \cite{KPRU14}, where personalized outputs to each user need not be private with respect to their own data, only to the data of others.
Thus databases $D=(X,\bm{y})$ and $D'=(X',\bm{y})$ are considered neighboring in our setting if $X$ and $X'$ differ in at most one row and have the same target unit $\bm{y}$.

\begin{definition}[Differential privacy \cite{DMNS06}]
A randomized algorithm $\mathcal{M}$ with domain $\mathcal{D}$ is $(\eps, \delta)$-differentially private for all $\mathcal{S} \subseteq \mbox{Range}(\mathcal{M})$ and for all pairs of neighboring databases $D,D' \in \mathcal{D}$,
\[
\Pr[\mathcal{M}(D) \in \mathcal{S} ] \leq \exp (\eps)
\Pr[\mathcal{M}(D') \in \mathcal{S} ] + \delta,
\]
where the probability space is over the internal randomness of the mechanism $\mathcal{M}$. If $\delta =0$, we say $\mathcal{M}$ is $\eps$-differentially private.
\end{definition}

\begin{definition}[$\ell_2$ sensitivity]
The \emph{$\ell_2$ sensitivity} of a vector-valued function $\f$, denoted $\Delta \f$, is the maximum $\ell_2$-norm change in the function's value between neighboring databases:
\[
\Delta \f =  \max_{D,D' \text{ neighbors}} || \f(D) - \f(D') ||_2.
\]
\end{definition}

A common method for achieving $\eps$-differential privacy for vector-valued functions is the \emph{high-dimensional Laplace Mechanism} \cite{CMS11}, which privately evaluates a function $\f$ on a dataset $D$ by first evaluating $\f(D)$ and then adding a Laplace noise vector $\bm{v}$ sampled according to density 
$p(\bm{v};a) \propto \exp \left( -\frac{ ||\bm{v}||_2}{a} \right)$, with parameter $a = \frac{\Delta \f}{\eps}$.
Note that this an extension of the (single-dimensional) Laplace Mechanism \cite{DMNS06} which would add Laplace noise with parameter $\Delta f/\eps$ to achieve $\eps$-DP for real-valued queries. Alternatively, one can add Gaussian noise of mean 0 and standard deviation at least $\sqrt{2 \ln(1.25/\delta)}\Delta f/\epsilon$ to achieve $(\epsilon,\delta)$-DP.

Differential privacy is robust to \emph{post-processing}, meaning that any downstream computation performed on the output of a differentially private algorithm will retain the same privacy guarantee. DP also \emph{composes}, meaning that if an $(\eps_1,\delta_1)$-DP mechanism and an $(\eps_2,\delta_2)$-DP mechanism are performed on the same database, then the entire process is $(\eps_1+\eps_2,\delta_1+\delta_2)$-DP.

\section{Differentially Private Synthetic Control (DPSC) Algorithms}\label{s.algo}

In this section, we present two algorithms for differentially private synthetic control, $DPSC_{out}$ (Algorithm \ref{alg.output}) and $DPSC_{obj}$ (Algorithm \ref{alg.obj}). 
Similar to non-private synthetic control algorithms (e.g., Algorithm \ref{alg.synth}), both algorithms are divided into two high-level steps: first the algorithm learns an estimate of the regression coefficients $\f$, and then it uses these coefficients to predict the post-intervention target unit $\bm{y_{post}}$. To ensure differential privacy of the overall algorithm, both of these steps must be performed privately.
The second step remains the same for both, and only the first part differs: $DPSC_{out}$ adds privacy noise directly to the output of the algorithm (output perturbation), whereas $DPSC_{obj}$ perturbs the objective function and minimizes the noisy loss function (objective perturbation). In the following subsections, we present and explain both algorithms.

\subsection{DPSC via Output Perturbation $DPSC_{out}$}
\label{s.algo.output}

Our first algorithm is $DPSC_{out}$ (Algorithm \ref{alg.output}), which utilizes \emph{output perturbation} to achieve differential privacy.
The learning step of this algorithm formalizes synthetic control as an instance of empirical risk minimization with the Ridge regression loss function given in Equation \eqref{eq.scloss}. This enables us to apply the Output Perturbation method of \cite{CMS11} for differentially private ERM, instantiated as the high-dimensional Laplace Mechanism.
The algorithm first learns the optimal (empirical risk minimizing) non-private regression coefficients $\f^{reg}$ as in Algorithm \ref{alg.synth}. It then samples a noise vector $\bm{v}$ from a high-dimensional Laplace distribution with parameter $\Delta \f^{reg}/\eps_1$, as described in Section \ref{s.dp}. Finally, the privatized regression coefficient vector is $\f^{out} = \f^{reg} + \bm{v}$.

The prediction step uses this coefficient vector to predict $\bm{y_{post}}$. A simple approach would be to directly predict $\bm{\hat{y}_{post}} = X_{post}^{\top}\f^{out}$; however, this approach would not provide privacy for the post-intervention donor data $X_{post}$. Instead, we again apply the high-dimensional Laplace Mechanism to privatize $X_{post}$ by adding a noise matrix $W$ sampled from a high-dimensional Laplace distribution with parameter $\Delta X_{post}/\eps_2$. The  privatized version of donor data is $\tilde{X}_{post}=X_{post}+W$, which is then used along with the privatized regression coefficients to produce the private prediction of the post-intervention target unit: $\bm{y}^{out} = \tilde{X}_{post}^{\top}\f^{out}$.

\begin{algorithm}[tbh]
\caption{DPSC via Output Perturbation $DPSC_{out}(X_{pre}, X_{post}, \bm{y_{pre}}, n, T, T_0, \lambda, \eps_1, \eps_2$)}\label{alg.output}
\begin{algorithmic}
\State \textbf{Step 1: Learn regression coefficients} 
\State Learn the regression coefficient $\f^{reg}$ using Ridge regression with parameter $\lambda \geq 0$:
\[\f^{reg} = \argmin_{\f\in \mathbb{R}^{n}} \frac{1}{T_0} ||\bm{y_{pre}} - {X}_{pre}^{\top} \f||_2^2 + \frac{\lambda}{2T_0} ||\f||_2^2.\]
\State Let $a=\frac{\Delta \f^{reg} }{\eps_1} = \frac{4T_0\sqrt{8 + n}}{ \lambda \eps_1}$
\State Sample $\bm{v}\in \mathbb{R}^{n}$ according to pdf $p(\bm{v};a) \propto \exp \left( -\frac{ ||\bm{v}||_2}{a} \right)$
\State Let $\f^{out} = \f^{reg} + \bm{v}$
\vspace{3mm}
\State \textbf{Step 2: Predict $\bm{y_{post}}$ via projection}
\State Let $b = \frac{2\sqrt{T-T_0}}{\eps_2}$
\State Sample each entry of $W \in \mathbb{R}^{n\times (T-T_0)}$ i.i.d.~according to pdf $p(W;b) \propto \exp \left( -\frac{||W||_F}{b} \right)$
\State Let $\Tilde{X}_{post} = X_{post} + W $
\State Output $\bm{y}^{out} = \Tilde{X}_{post}^{\top} \f^{out}.$
\end{algorithmic}
\end{algorithm}

The entire algorithm is then $(\eps_1 + \eps_2,0)$-differentially private by composition of these two steps. 
We remark that the algorithm does not output $\f^{out}$, simply because this vector is typically not of interest in most synthetic control problems, and is instead considered only an intermediate analysis step. However, this vector could be output if desired with no additional privacy loss because Step 1 of the $DPSC_{out}$ algorithm is $\eps_1$-differentially private (Theorem \ref{step1.privacy}), and this privacy loss is already accounted for in the composition step.

We provide two main results on the privacy and accuracy of  $DPSC_{out}$. First, Theorem \ref{thm.priv.out} shows that our algorithm is differentially private. Although our algorithm relies on algorithmic techniques from \cite{CMS11} for differentially private ERM, the vertical regression setup in synthetic control requires novel sensitivity analysis for $\f^{reg}$, which constitutes the bulk of the work required to prove Theorem \ref{thm.priv.out}.
Theorem \ref{thm.acc.out} shows that our $DPSC_{out}$ algorithm produces an accurate prediction of the post-intervention target unit, as measured by the standard metric (e.g., \cite{rsc}) of root mean squared error (RMSE) with respect to the true signal vector $\bm{m}$. In Section \ref{s.accdist}, we also extend Theorem \ref{thm.acc.out} to to remove the dependence on distributional parameters and provide an expression of RMSE that depends only on the input parameters, under some mild additional assumptions on the distribution of data.
Full proofs for Theorems \ref{thm.priv.out} and \ref{thm.acc.out} are respectively presented in Sections \ref{s.privacy} and \ref{s.accuracy}. 

\begin{theorem}\label{thm.priv.out}
$DPSC_{out}$ of Algorithm \ref{alg.output} is $(\eps_1+\eps_2,0$)-differentially private.
\end{theorem}

\begin{restatable}{theorem}{mainacc}
\label{thm.acc.out}
The estimator $\bm{y}^{out}$ output by Algorithm \ref{alg.output} satisfies: 
\[
RMSE(\bm{y}^{out})
\leq \frac{||M_{post}||_2}{\sqrt{T-T_0}} \left( \mathbb{E}[||\f^{reg}-\f||_2]  + \frac{4T_0 \sqrt{8+n}}{\lambda \eps_1} \right) 
+\left(\sqrt{n \sigma^2} + \frac{\sqrt{2}}{\eps_2}\right) \left(\sqrt{n} \psi + \frac{4T_0 \sqrt{8+n}}{\lambda \eps_1} \right),
\]
where $||\f^{reg}||_{\infty} \leq \psi$ for some $\psi>0$, and RMSE is the average root mean squared error of the estimator, defined as $RMSE(\bm{y}^{out}) = \frac{1}{\sqrt{T-T_0}} \mathbb{E}[||\bm{y}^{out}-\bm{m}_{post}||_2]$.
\end{restatable}

\begin{remark}
The accuracy bound grows as $O(n)$, which is shown to be necessary in Section \ref{s.ndepend}. While this might be undesirable in most other learning domains, $n$ does not grow with the problem size in synthetic control settings for a few reasons. Typically, $M$ is assumed to be a low-rank matrix and hence $X$ is \emph{approximately} low rank \cite{rsc, mrsc}. This is not only an assumption, but true in most cases \cite{udell2019lowrank}. Therefore, there exists a saturation point where adding additional donors does not meaningfully improve accuracy. 
In practice, donors must be  carefully selected the donors to maintain the low rank condition, and finding a way to select appropriate donors remains an active area of research in synthetic control \cite{abadie2010synthetic, dube2015pooling}.
Additionally, synthetic control should be viewed as a regression problem with $T_0$ data points in $\mathbb{R}^n$, so $n$ is the dimension of the data rather than number of samples. 
The remaining dependence of the accuracy guarantee on $T_0$ can be handled by setting $\lambda=O(T_0)$ (see Section \ref{s.privacycost} for details).

\end{remark}

\subsection{DPSC via Objective Perturbation $DPSC_{obj}$}\label{s.algo.obj}

We next present our second algorithm for differentially private synthetic control, $DPSC_{obj}$ (Algorithm \ref{alg.obj}), based on objective perturbation. While Step 2 remains unchanged relative to Algorithm \ref{alg.output}, Step 1 is modified to perturb the objective function itself and then exactly optimize the perturbed objective, instead of first computing the optimal non-private coefficients and then adding noise.
Objective perturbation has been shown to outperform output perturbation in the standard private ERM setting when the loss function is strongly convex \cite{CMS11}.

The algorithm augments the objective function with two terms. The first is an additional regularization term to ensure $\frac{\lambda+\Delta}{T_0}$-strong convexity (compared to $\frac{\lambda}{T_0}$ as the regularization term of Algorithm \ref{alg.output}). The $\Delta$ parameter is tuned by the algorithm to ensure that it can still satisfy $(\eps_1,\delta)$-DP in Step 1, even when $\eps_1$ is small. The second is the noise term $\bm{b}^{\top}\bm{f}$ to ensure privacy, where $\bm{b}$ is sampled from a high-dimensional Laplace distribution if $(\eps,0)$-DP is desired (i.e., if $\delta=0$), and from a multi-variate Gaussian distribution if $(\eps,\delta)$-DP is desired (i.e., if $\delta>0$). 

The algorithm then exactly optimizes this new objective function, where the noise term $\bm{b}$ ensures that this minimization satisfies differential privacy. Although the algorithmic procedure in Step 1 is similar to that of Objective Perturbation algorithms for DP-ERM of \cite{CMS11, KST12}, the sensitivity and privacy analysis again requires substantial novelty because the definition of neighboring databases is different, and previous work cannot be immediately applicable to the transposed regression setting.
Finally, Algorithm \ref{alg.obj} maintains the same Step 2 process as Algorithm \ref{alg.output} to predict $\bm{y_{post}}$, based on $\f^{obj}$ computed from Step 1. Algorithm \ref{alg.obj} is $(\eps_1 + \eps_2,\delta)$-differentially private by composition of privacy guarantees from these two steps.

\begin{algorithm}[tbh]
\caption{DPSC via Objective Perturbation $DPSC_{obj}(X_{pre}, X_{post}, y_{pre}, n, T, T_0, \lambda, \eps_1, \eps_2, \delta, c$)}\label{alg.obj}
\begin{algorithmic}
\State \textbf{Step 1: Learn regression coefficients}
\If{$\eps_1 > \log(1+\frac{2c}{\lambda} + \frac{c^2}{\lambda^2})$}
    \State Let $\eps_0 = \eps_1 - \log(1+\frac{2c}{\lambda} + \frac{c^2}{\lambda^2})$ and $\Delta = 0$
\Else
    \State $\eps_0 = \frac{\eps_1}{2}$ and $\Delta = \frac{c}{e^{(\eps_1/4)}-1}-\lambda$
\EndIf
\If{$\delta>0$}
    \State Sample $\bm{b}\in \mathbb{R}^{n}$ according to $\mathcal{N}(0, \beta^2 I_n)$, where $\beta = \frac{4 T_0 \sqrt{8+n} \sqrt{2 \log \frac{2}{\delta} + \eps_0}}{\eps_0}$
\Else
    \State Sample $\bm{b}\in \mathbb{R}^{n}$ according to pdf $p(\bm{b};\beta) \propto \exp \left( -\frac{ ||\bm{b}||_2}{\beta} \right)$, where $\beta = \min\{ \frac{4 T_0 \sqrt{8+n}}{\eps_0}, \frac{c\sqrt{n}+4T_0}{\eps_0}\}$
\EndIf
\State Learn the regression coefficient ${\f}^{obj}$ using parameter $\lambda \geq 0$
\[{\f}^{obj} = \argmin_{\f\in \mathbb{R}^{n}} \frac{1}{T_0} ||y_{pre} - {X}_{pre}^{\top} \f||_2^2 + \frac{\lambda + \Delta}{2T_0} ||\f||_2^2 + \frac{1}{T_0}\bm{b}^{\top}\f.\]
\State \textbf{Step 2: Predict $\bm{y_{post}}$ via projection}

\State Let $b = \frac{2\sqrt{T-T_0}}{\eps_2}$
\State Sample each entry of $W \in \mathbb{R}^{n\times (T-T_0)}$ i.i.d.~according to pdf $p(W;b) \propto \exp \left( -\frac{||W||_F}{b} \right)$
\State Let $\Tilde{X}_{post} = X_{post} + W $
\State Output $\bm{y}^{obj} = \Tilde{X}_{post}^{\top} \f^{obj}$
\end{algorithmic}
\end{algorithm}

$DPSC_{obj}$ requires an additional parameter $c$ that is used in the analysis to bound the maximum absolute eigenvalue of $2(X'_{pre}X'^{\top}_{pre} - X_{pre}X^{\top}_{pre})$, which is closely related to $||\nabla\cL(\f)||_2$. Because $X_{pre}$ and $X'_{pre}$ are neighboring databases, then the matrix of interest will only have one column and one row that are non-zero. In our setting, we use the fact that all entries of $X$ are bounded between $-1$ and $1$ to derive an upper bound on this matrix and its eigenvalues. In general, an analyst can use domain expertise or prior knowledge of the data distribution to choose an appropriate value of $c$. Additional details and guidance for choosing $c$ can be found in Appendix \ref{app.pickc}.

We provide two main results on the privacy and accuracy of $DPSC_{obj}$. First, Theorem \ref{thm.priv.obj} shows that our algorithm is differentially private. To prove privacy of Step 1, we must consider two cases based on the value of $\Delta$, which adds additional strong convexity to the loss function if it is needed. The privacy budget must be allocated differently within the analysis in the two cases of $\Delta=0$ and $\Delta>0$.

Theorem \ref{thm.acc.obj} shows that $DPSC_{obj}$ produces an accurate prediction of the post-intervention target unit, as measured by RMSE between its output $\bm{y}^{obj}$ and the target unit's post-intervention signal vector $\bm{m_{post}}$. As with $DPSC_{out}$, we also extend Theorem \ref{thm.acc.obj} in Section \ref{s.objclosedform} to provide an explicit closed-form bound on RSME that does not depend on the distributional parameters. Full proofs for for Theorems \ref{thm.priv.obj} and \ref{thm.acc.obj}, along with their extensions, are respectively presented in Sections \ref{s.privacyobj} and \ref{s.accuracyobj}.

\begin{theorem}\label{thm.priv.obj}
$DPSC_{obj}$ of Algorithm \ref{alg.obj} is $(\eps_1+\eps_2,\delta$)-differentially private.
\end{theorem}

\begin{restatable}{theorem}{mainaccobj}
\label{thm.acc.obj}
The estimator $\bm{y}^{obj}$ output by Algorithm \ref{alg.obj} satisfies: 
\begin{equation*}
    \begin{split}
        RMSE(\bm{y}^{obj})
&\leq \frac{||M_{post}||_2}{\sqrt{T-T_0}} \left( \E[||(\f^{reg}-\f) ||_2] + \frac{2}{\lambda+\Delta}\E[||\bm{b}||_2] + \mathds{1}_{\Delta \neq 0} \left( \frac{1}{\lambda} + \frac{1}{\lambda+\Delta} \right) 2T_0^2 \sqrt{n} \right)\\
&\quad + \left( \sqrt{n \sigma^2 } + \frac{\sqrt{2}}{\eps_2} \right) \left( \sqrt{n} \psi + \frac{2}{\lambda+\Delta}\E[||\bm{b}||_2] + \mathds{1}_{\Delta \neq 0} \left( \frac{1}{\lambda} + \frac{1}{\lambda+\Delta} \right) 2T_0^2 \sqrt{n} \right),
    \end{split}
\end{equation*}
where $||\f^{reg}||_{\infty} \leq \psi$ for some $\psi>0$, and $\E[||\bm{b}||_2] = \sqrt{ \frac{n T_0 4\sqrt{8+n} \sqrt{2 \log \frac{2}{\delta} + \eps_0}}{\eps_0} }$ for Gaussian noise ($\delta>0$ case) and $\E[||\bm{b}||_2] = \min\{ \frac{4 T_0 \sqrt{8+n}}{\eps_0}, \frac{c\sqrt{n}+4T_0}{\eps_0}\}$ for Laplace noise ($\delta=0$ case), and $\eps_0$, and $\Delta$ are computed internally by the algorithm.
\end{restatable}

As in Section \ref{s.algo.output}, we remark that while the accuracy bound of Theorem \ref{thm.acc.obj} grows as $O(n)$, in our setting, $n$ does not typically grow substantially with the problem size, both in theory \cite{rsc, mrsc} and in practice \cite{udell2019lowrank}.

\subsection{Comparison between DP-ERM and DPSC}\label{s.comparison}

Before proving our main privacy and accuracy theorems in the following sections, we first briefly compare the results of DP-ERM \cite{CMS11} with our approach. Consider a Ridge regression task in $p$-dimensional space with $q$ samples (i.e., covariates $\bm{x}_k \in \R^p$ and labels $y_k \in \R$, $\forall k\in[q]$). The regression coefficient $\f \in \R^p$ is learned by a standard empirical risk minimization process with a regularizer $\lambda||\f||_2^2$.

In the traditional regression setup where the privacy goal is to protect one sample $\bm{x}_k$---corresponding to one individual's data---the sensitivity of the regression task is $\Delta\f=\frac{2}{q\lambda}$ \cite{CMS11}. It does not depend on the dimension $p$, and the sensitivity decreases as the number of samples $q$ increases. Intuitively, adding or removing one person's data should exhibit diminishing marginal effect on the final model $\f$ as the training sample size grows.

On the other hand, in our transposed setting of synthetic control, the privacy goal is to protect the $i$-th entry of each $\bm{x}_k$ (i.e., an individual's data are spread across all samples), the sensitivity is  $\Delta\f=\frac{4q\sqrt{8+p}}{\lambda}$ (Lemma \ref{lem.sens}).
In this setting, each dimension of the coefficient $\f$ captures how important the corresponding donor is for explaining the target; hence the impact of changing $i$-th person's data will have a significant on the $i$-th dimension of $\f$, regardless of the number of individuals in the donor pool. This difference is at the crux of why it is more difficult to guarantee privacy in the transposed setting of synthetic control, relative to the standard regression setting.

\section{Privacy Guarantees of $DPSC_{out}$}\label{s.privacy}

In this section, we will prove Theorem \ref{thm.priv.out}, that DPSC is $(\eps_1 + \eps_2,0)$-differentially private.
The proof relies on the privacy of $\f^{out}$ in the learning phase, and then $\tilde{X}_{post}$ in the prediction phase. At a high level, $\f^{out}$ is $\eps_1$-DP through a (non-trivial) application of the Output Perturbation algorithm of \cite{CMS11}. The non-triviality comes from the vertical regression used in synthetic control, rather than the horizontal regression classically used in empirical risk minimization (as illustrated in Figure \ref{fig.synthcon}), which requires novel sensitivity analysis of the function $\f^{reg}$. In the prediction phase, we must show that sufficient noise is added to ensure $\tilde{X}_{post}$ is an $\eps_2$-DP version of $X_{post}$. The final privacy guarantee of $\bm{y}^{out}$ comes from post-processing and composition of these two private estimates.

\subsection{Privacy of ${\emph{\textbf{f}}}^{out}$}

Let us begin by proving that $\f^{out}$ is $\eps_1$-DP. 

\begin{theorem}\label{step1.privacy}
Step 1 of Algorithm \ref{alg.output} that computes $\f^{out}$ is $(\eps_1,0)$-differentially private.
\end{theorem}

It might seem that Theorem \ref{step1.privacy} should follow immediately from the privacy guarantees of Output Perturbation in \cite{CMS11}. Indeed, Theorem 6 of \cite{CMS11} states that a similar algorithm is $(\eps,0)$-DP under certain technical conditions. However, the proof of this result relies on sensitivity analysis of classical empirical risk minimization (see Corollary 8 of \cite{CMS11}) which does not hold in the synthetic control setting. The crux of the difference comes from the vertical regression (i.e., along the columns) of synthetic control as illustrated in Figure \ref{fig.synthcon}, while privacy must still be maintained along the rows. Thus the sensitivity of $\f^{reg}$ to changing in a single $\emph{donor row}$ is fundamentally different from the sensitivity in a standard empirical risk minimization setting, as explained in Section \ref{s.comparison}. See Remark \ref{rem.compare} for a more technical exploration of this difference. Additionally, while the ERM framework of \cite{CMS11} is fully general, their results (including Theorem 6 and Corollary 8) apply only to the problem setting of binary classification via logistic regression, by assuming a specific loss function $\cL$ in the analysis.

Instead, we prove Theorem \ref{step1.privacy} primarily using first-principles (i.e., direct sensitivity analysis and the Laplace Mechanism of \cite{DMNS06}, which also underpins the results of \cite{CMS11}) starting with Lemma \ref{lem.sens}.  The proof of Lemma \ref{lem.sens} and Theorem \ref{step1.privacy} will be augmented with one intermediate result for output perturbation from \cite{CMS11} that does apply to our setting, and one fact from \cite{truthful}, which extended the binary classification result of \cite{CMS11} to the Ridge regression loss function that we use.

\begin{lemma}\label{lem.sens}
The $\ell_2$ sensitivity of $\f^{reg}$ is
\[\Delta \f^{reg} \leq \frac{4T_0\sqrt{8 + n}}{ \lambda}.\]
\end{lemma}

To prove Lemma \ref{lem.sens}, we will first use the following lemma from \cite{CMS11}, which bounds the sensitivity of $\f^{reg}$ as a function of the strong convexity parameter of the loss function $\cL$.

\begin{lemma}[\cite{CMS11}, Lemma 7]\label{lem.cms}
Let $G(\f)$ and $g(\f)$ be two vector-valued functions, which are continuous, and differentiable at all points. Moreover, let $G(\f)$ and $G(\f) + g(\f)$ be $\lambda$-strongly convex. If $\f_1 = \argmin_{\f} G(\f)$ and $\f_2 = \argmin_{\f} G(\f) +g(\f)$, then
\[ \|\f_1 - \f_2\|_2 \leq \frac{1}{\lambda}\max_{\f} \| \nabla g(\f)\|_2.\]
\end{lemma}

We instantiate this lemma by defining 
\begin{equation}\label{eq.g}
  G(\f) = \cL(\f,\mathcal{D}) \quad \text{and} \quad g(\f) = \cL(\f,\mathcal{D'})
- \cL(\f,\mathcal{D}),  
\end{equation}
for two arbitrarily neighboring databases $\mathcal{D},\mathcal{D}'$ and defining the following two maximizers:
\[\bm{f_1} = \argmin \cL(\f,\mathcal{D}) = \argmin G(\f)
\quad \text{ and } \quad
\bm{f_2} =  \argmin  \cL(\f,\mathcal{D'}) = \argmin G(\f)+g(\f).
\]
Then,
\[
\Delta \f^{reg} = \max_{\mathcal{D},\mathcal{D}' \text{ neighbors}} \|\f_1 - \f_2\|_2.
\]
To apply Lemma \ref{lem.cms}, we must show that $G(\f)$ and $g(\f)$ are continuous and differentiable. $G(\f)$ is simply the Ridge regression loss function, which is known to be continuous and differentiable \cite{hastie2009elements}. 
Since $g(\f)$ is the difference between two continuous and differentiable functions, then it is also continuous and differentiable \cite{boyd2004convex}.
We must also show strong convexity of $G(\f)$ and $G(\f)+g(\f)$. The following lemma from \cite{truthful} immediately gives that these two functions are both strongly convex.

\begin{lemma}[\cite{truthful}, Lemma 32]\label{lem.convex}
The Ridge regression loss function with regularizer $\frac{\lambda}{2T_0}$ is $\frac{\lambda}{T_0}$-strongly convex.
\end{lemma}
Thus by Lemma \ref{lem.cms}, the sensitivity $\Delta \f^{reg} = \max_{\mathcal{D},\mathcal{D}' \text{ neighbors}} \|\f_1 - \f_2\|_2 \leq \frac{T_0}{\lambda}\max_{\f} \| \nabla g(\f)\|_2$.  All that remains is to bound $\| \nabla g(\f)\|_2$. A proof of the following lemma is deferred to Appendix \ref{app.proofs}.

\begin{restatable}{lemma}{goff}\label{lem.goff}
Let $g(\f) = \cL(\f,\mathcal{D'})
- \cL(\f,\mathcal{D})$ for two arbitrarily neighboring databases $\D,\D'$. Then,
\[\max_{\f} \| \nabla g(\f)\| \leq 4\sqrt{8 + n}.\]
\end{restatable}

\begin{remark}\label{rem.compare}
If we were instead considering simple linear regression in the classical setting (i.e., as in \cite{CMS11}) using $T_0$ data points with $n$ dimensional features, $g(\f)$ would only contain one term in the difference between the losses, namely, the one data point ($\bm{x_i}, y_i$) that differs across the two neighboring databases. This yields
\[
g(\f) = \frac{1}{T_0}((\bm{x_i}' -\bm{x_i})^{\top} \f - (y_i'-y_i))^2
\]
with gradient
\[
\nabla g(\f) = \frac{2}{T_0}((\bm{x_i}'-\bm{x_i})(\bm{x_i}'-\bm{x_i})^{\top} \f - (y_i'-y_i) (\bm{x_i}'-\bm{x_i})),
\]
which can be bounded by $O(\frac{1}{T_0})$.
This result does not depend on the dimension of the features ($n$) and only depends on the number of data points ($T_0$).

However, in synthetic control, terms do not cancel as neatly across neighboring databases, and instead, 
\[ g(\f) =\frac{1}{T_0}\sum_{t=1}^{T_0} \left[
        \left( \bm{x_t}^{\top} \f -y_t \right) +
        \left( \bm{x_t}^{\top} \f  -x_{i,t}f_i + x'_{i,t}f_i -y_t \right)
        \right](x'_{i,t}-x_{i,t})f_i. \]
Through a more involved analysis of this expression, we get the bound of Lemma \ref{lem.goff}, which depends on $n$, rather than $T_0$.
\end{remark}

Using these lemmas, we can now bound the sensitivity of our query, to complete the proof of Lemma \ref{lem.sens}.

\begin{equation}\label{eq.sens.f}
    \Delta \f^{reg} =  \max_{\mathcal{D}, \mathcal{D'} \mbox{ neighbors}} || \f(\mathcal{D}) - \f(\mathcal{D'}) ||_2  \leq \max || \bm{f_1} -\bm{f_2} || \leq \frac{4 T_0 \sqrt{8 + n}}{ \lambda}.
\end{equation}

Theorem \ref{step1.privacy} then follows from the privacy guarantee of the high-dimensional Laplace Mechanism instantiated with the appropriate sensitivity value.

\subsubsection{Dependence on $n$}\label{s.ndepend}

One might wonder whether the asymptotic dependence on $n$ and $T_0$ in the sensitivity is necessary. In practice, one should set $\lambda=O(T_0)$ (as discussed in greater detail in Section \ref{s.privacycost}), so the dependence on $T_0$ will not affect the accuracy of the algorithm.  However, as we show next in Lemma \ref{lem.sens.lowerbnd}, the dependence on $n$ is asymptotically tight.

\begin{lemma}\label{lem.sens.lowerbnd}
The $\ell_2$ sensitivity of $\f^{reg}$ is $\Delta \f^{reg} = \Omega(\sqrt{n})$.
\end{lemma}
\begin{proof}
Consider two neighboring databases $(X,y)$ and $(X',y)$, where $\bm{y} = \bm{1} \in \mathbb{R}^{T_0}$, $X \in \mathbb{R}^{n \times T_0}$ has all entries $1/n$, except the first row, which is all 1s. Neighboring database $X'$ differs from $X$ only in the first row, which is instead all 0s, and all other entries and $1/n$. The dimensions in this example are chosen to be $T_0=n$, and we choose $\lambda=2T_0$, so that the regularization coefficient is $1$.

Computing the minimizers of the loss functions under each neighboring database using the closed-form expression yields $\f^{reg}=(XX^\top + I)^{-1}X\bm{y}$ with first coordinate equal to $\frac{n^2}{n^2+2n-1}$, and all other coordinates are $\frac{n}{n^2+2n-1}$, and ${\f^{reg}}'=({X'X'}^\top + I)^{-1}X'\bm{y}$ with first coordinate $0$ and all other coordinates $\frac{-n}{1-2n}$. This yields $\ell_2$ difference of,
\[
||\f^{reg}-{\f^{reg}}'||_2 = \sqrt{ \left( \frac{n^2}{n^2+2n-1} \right)^2 + (n-1) \left( \frac{n^3}{(n^2+2n-1)(1-2n)} \right)^2} = \Theta(\sqrt{n}).
\]
Since we have a pair of neighboring databases with $\ell_2$ distance in their output of $\Theta(\sqrt{n})$, then the sensitivity of $\f^{reg}$ cannot be $o(\sqrt{n})$.
\end{proof}

\begin{remark}
We note that while the example in Lemma \ref{lem.sens.lowerbnd} is mathematically valid, such a degenerate case where all the donors are identical except for one person and the (exact) rank of the donor matrix is 1 is unlikely to happen in practical settings. Thus suggests that with additional domain knowledge on the selection criteria for donors, practitioners may be able to reduce the sensitivity and thus add less noise for privacy in special restricted cases of interest.
\end{remark}

\subsection{Privacy of $\tilde{X}_{post}$ and ${\emph{\textbf{y}}}^{out}$}

Next we move to privacy of $\Tilde{X}_{post}$ and its role in ensuring privacy of $\bm{y}^{out}$.

\begin{lemma}\label{step2.privacy}
The computation of $\Tilde{X}_{post}$ in Step 2 of Algorithm \ref{alg.output} is ($\eps_2,0$)-differentially private.
\end{lemma}

$\Tilde{X}_{post}$ is privatized through a simple application of the Laplace Mechanism of \cite{DMNS06}. Thus to prove Lemma \ref{step2.privacy}, we need only to bound the sensitivity of $X_{post}$ to show that the algorithm adds sufficient noise.
We first note that the Frobenius norm of a matrix $X \in \mathbb{R}^{n \times (T-T_0)}$ is equal to the $\ell_2$ norm of the equivalent flattened vector $X \in \mathbb{R}^{n (T-T_0)}$ \cite{HJ12}. Thus implementing the matrix-valued Laplace Mechanism with noise parameter calibrated to the $\ell_2$ sensitivity of the flattened matrix-valued query over $\eps$ will ensure $(\eps,0)$-differential privacy.

\begin{lemma}\label{lem.sens2}
The $\ell_2$ sensitivity of flattened ${X}_{post}$ is $2\sqrt{(T-T_0)}$.
\end{lemma}
\begin{proof}
Changing one donor unit in $X_{post}$ can change at most $T-T_0$ entries in the matrix. Since all entries in $X_{post}$ are bounded in $[-1,1]$, each data point can change by at most $2$ between two neighboring databases. Thus viewing $X_{post}$ as a flattened matrix, this will change the $\ell_2$-norm of $X_{post}$ by at most $2\sqrt{(T-T_0)}$.
\end{proof}

Finally, we can combine Theorem \ref{step1.privacy} and Lemma \ref{step2.privacy} to complete the proof of Theorem \ref{thm.priv.out}. The estimates $\f^{out}$ and $\tilde{X}_{post}$ are together $(\eps_1 + \eps_2,0)$-differentially private by DP composition, and then $\bm{y}^{out}$ is $(\eps_1 + \eps_2,0)$-differentially private by post-processing. We note that if one wanted to publish $\f^{out}$, this would not incur any additional privacy loss.

\section{Accuracy Guarantees of $DPSC_{out}$}\label{s.accuracy}

In this section we will analyze the accuracy of $DPSC_{out}$. We first prove Theorem \ref{thm.acc.out}, restated below for convenience.

\mainacc*

This theorem gives bounds on the predicted post-intervention target vector $\bm{{y}^{out}}$, as measured by RMSE.
This result is stated in full generality with respect to the distribution of data and the latent variables, and thus the bound depends on terms such as $||M_{post}||_{2,2}$ and $\mathbb{E}[||\f^{reg}-\f||_2]$. This is consistent with comparable bounds on the RMSE of robust synthetic control \cite{rsc} which also depended on these terms (although the stated bounds of \cite{rsc} suppress dependence on $n$).
Section \ref{s.accpost.out} provides a proof of this main result, with omitted detailed deferred to Appendix \ref{app.proofs}.

Analysts may still wonder about the full asymptotic performance of $DPSC_{out}$ algorithm. To this end, in Section \ref{s.accdist}, we additionally derive closed-form bounds for these distribution-dependent terms (under some mild assumptions).  We present Corollary \ref{cor.accuracy}, which gives a bound on RMSE of $\bm{y}^{out}$ that depends only on input parameters of the algorithm and the model.

\subsection{Accuracy of post-intervention prediction $\emph{y}^{out}$}\label{s.accpost.out}

We will prove Theorem \ref{thm.acc.out} by showing that the prediction vector $\bm{y}^{out}$ output by $DPSC_{out}$ in Algorithm \ref{alg.output} is close to the true values, as measured by Root Mean Squared Error (RMSE), defined as follows:
\begin{equation}\label{eq.rmse}
RMSE(\bm{y}^{out}) = \frac{1}{\sqrt{T-T_0}}\mathbb{E}[||\bm{y}^{out}-\bm{m}_{post}||_2].
\end{equation}

We note that while it may seem most natural to bound the difference between $\bm{y}^{out}$ and $\bm{y}_{post}$, we instead use $\bm{m}_{post}$ for two reasons. Firstly, $\bm{y}_{post}$ may not even match $\bm{y}_{pre}$ due to the intervention. Secondly, $\bm{m}_{post}$ captures the true signal that we are trying to estimate, which is the counterfactual outcome without the intervention.

We begin by bounding the expected $\ell_2$ difference between $\bm{y}^{out}$ and $\bm{m}_{post}$. Using the fact that 
\[\bm{y}^{out} = \tilde{X}_{post}^{\top} \bm{f}^{out} = (X_{post}^{\top} + W^{\top})(\f^{reg} + \bm{v}),\]
and that $X_{post} = M_{post}+Z_{post}$ (by Equation \eqref{eq.noise}) and $\bm{m} = M_{post}^{\top}\f$ (by Equation \eqref{eq.model}), we can expand the expectation as follows:
\begin{align}\label{eq.accuracy}
    \mathbb{E}[||\bm{y}^{out}-\bm{m}_{post}||_2]
    &= \mathbb{E}[|| (X_{post}^{\top} + W^{\top})(\f^{reg} + \bm{v}) - M_{post}^{\top}\f ||_2]\notag\\
    &= \mathbb{E}[|| (M_{post}^{\top} +Z_{post}^{\top} + W^{\top})(\f^{reg} + \bm{v}) - M_{post}^{\top}\f ||_2]\notag\\
    &\leq \mathbb{E}[|| M_{post}^{\top}(\f^{reg} -\f)||_2 +||(Z_{post}^{\top} + W^{\top}) \f^{reg} ||_2 +  ||(M_{post}^{\top}+Z_{post}^{\top} + W^{\top})\bm{v}||_2 ]\notag\\
    &= \mathbb{E}[|| M_{post}^{\top}(\f^{reg} -\f)||_2] +\mathbb{E}[||(Z_{post}^{\top} + W^{\top}) \f^{reg} ||_2] +  \mathbb{E}[||(M_{post}^{\top}+Z_{post}^{\top} + W^{\top})\bm{v}||_2 ]
\end{align}

We next proceed to bound each of the terms in Equation \eqref{eq.accuracy} separately, making use of the following submultiplicative norm property, which holds for any matrix $A$ and vector $\bm{x}$:
\begin{equation}\label{eq.normbounds}
||A\bm{x}||_2 \leq ||A||_2 ||\bm{x}||_2 \leq ||A||_F ||\bm{x}||_2,
\end{equation}
where $||A||_2 = ||A||_{2,2}$ is the spectral norm of $A$, $||A||_F$ is the Frobenius norm of $A$, and $||\bm{x}||_2$ is the $\ell_2$ norm of $\bm{x}$.

We also know the distribution of the norms of noise terms $\bm{v}$ and $W$ that were added to preserve privacy, because they were constructed explicitly within Algorithm \ref{alg.output}:
\begin{equation}\label{eq.noisebounds}
\mathbb{E}[||\bm{v}||_2] = \frac{4T_0\sqrt{8+n}}{\lambda \eps_1} \quad \text{ and } \quad \mathbb{E}[||W||_F] = b = \frac{2\sqrt{T-T_0}}{\eps_2}.
\end{equation}

Using these facts, we can obtain bounds for the three terms in Equation \eqref{eq.accuracy}. A complete proof of Lemma \ref{lem.threebounds} can be found in Appendix \ref{app.proofthreebounds}.

\begin{restatable}{lemma}{threebounds}\label{lem.threebounds}
The three terms in Equation \eqref{eq.accuracy} can be bounded as follows:
\[
\mathbb{E}[|| M_{post}^{\top}(\f^{reg} -\f)||_2]
\leq ||M_{post}||_{2,2} \cdot \mathbb{E}[||\f^{reg} -\f||_2],
\]
\[
\mathbb{E}[||(Z_{post}^{\top} + W^{\top}) \f^{reg} ||_2]
\leq \sqrt{n} \psi \left(\sqrt{n(T-T_0) \sigma^2} + \frac{2\sqrt{T-T_0}}{\eps_2} \right), \text{ and}
\]
\[
\mathbb{E}[||(M_{post}^{\top}+Z_{post}^{\top} + W^{\top})\bm{v}||_2 ]
\leq \left(||M_{post}||_{2,2} + \sqrt{n(T-T_0) \sigma^2}+\frac{2\sqrt{T-T_0}}{\eps_2}\right) \frac{4T_0\sqrt{8+n}}{\lambda \eps_1}.
\]
\end{restatable}

Applying the bounds of Lemma \ref{lem.threebounds} to Equation \eqref{eq.accuracy} yields,
\begin{align*}
    \mathbb{E}[||\bm{y}^{out}-\bm{m}||_2]
    &\leq ||M_{post}||_{2,2} \cdot \mathbb{E}[\f^{reg} -\f||_2]  + \sqrt{n} \psi \left(\sqrt{n(T-T_0) \sigma^2} + \frac{2\sqrt{T-T_0}}{\eps_2} \right) \\
    & \quad + \left(||M_{post}||_{2,2} + \sqrt{n(T-T_0) \sigma^2} + \frac{2\sqrt{T-T_0}}{\eps_2}\right) \frac{4T_0\sqrt{8+n}}{\lambda \eps_1}\\
    &\leq ||M_{post}||_{2,2} \left( \mathbb{E}[\f^{reg} -\f||_2]+\frac{4T_0\sqrt{8+n}}{\lambda \eps_1} \right) + \left(\sqrt{n(T-T_0) \sigma^2} + \frac{2\sqrt{T-T_0}}{\eps_2} \right) \left( \sqrt{n} \psi + \frac{4T_0\sqrt{8+n}}{\lambda \eps_1}\right)
\end{align*}
Combining this with Equation \eqref{eq.rmse} gives the desired bound for Theorem \ref{thm.acc.out}:
\[
    RMSE(\bm{y}^{out})
    \leq \frac{||M_{post}||_{2,2}}{\sqrt{T-T_0}} \left( \mathbb{E}[||\f^{reg} -\f||_2]+\frac{4T_0\sqrt{8+n}}{\lambda \eps_1} \right)+ \left(\sqrt{n \sigma^2} + \frac{\sqrt{2}}{\eps_2} \right) \left( \sqrt{n} \psi + \frac{4T_0\sqrt{8+n}}{\lambda \eps_1}\right).
\]

\subsubsection{Cost of privacy in synthetic control}\label{s.privacycost}

To understand the additional error incurred due to privacy, compare the bound of Theorem \ref{thm.priv.out} to the RMSE of the equivalent non-private prediction, $\bm{y}^{reg} = X_{post}^{\top}\f^{reg}$.
\begin{align}\label{eq.rmse.yreg}
    RMSE(\bm{y}^{reg}) &= \frac{1}{\sqrt{T-T_0}} \mathbb{E}[||X_{post}^{\top}\f^{reg} - \bm{m}||_2] \notag\\
    &= \frac{1}{\sqrt{T-T_0}} \mathbb{E}[||(M_{post}^{\top}+Z_{post}^{\top})\f^{reg} - M_{post}^{\top}\f ||_2] \notag\\
    &\leq \frac{1}{\sqrt{T-T_0}} \mathbb{E}[||M_{post}^{\top} (\f^{reg} -\f) ||_2 + ||Z_{post}^{\top} \f^{reg}||_2 ] \notag\\
    &\leq \frac{||M_{post}||_{2,2}}{\sqrt{T-T_0}} \left(\mathbb{E}[||\f^{reg} -\f||_2] \right) + \sqrt{n}\psi \cdot \sqrt{n\sigma^2}
\end{align}
Lemma \ref{lem.estbounds} in the next section shows that $\mathbb{E}[||\f^{reg} -\f||_2] = O(\sqrt{n})$. Then the first term of Equation \eqref{eq.rmse.yreg} can be easily bounded using the following fact,
\[
||M_{post}||_{2,2} \leq ||M_{post}||_F \leq \sqrt{n(T-T_0)}, 
\]
so $\frac{||M_{post}||_{2,2}}{\sqrt{T-T_0}} \leq \sqrt{n}$. Thus we see that $RMSE(\bm{y}^{reg}) =O(n)$.

Comparing Equation \eqref{eq.rmse.yreg} with the bound on $RMSE(\bm{y}^{out})$ in Theorem \ref{thm.priv.out}, we observe that the additional terms induced by privacy are:
\begin{equation}\label{eq.accu.diff}
\frac{||M_{post}||_{2,2}}{\sqrt{T-T_0}} \frac{4T_0\sqrt{8+n}}{\lambda \eps_1} + 
\frac{4T_0\sqrt{(8+n)n \sigma^2}}{\lambda \eps_1} +  \frac{\sqrt{2n}\psi}{\eps_2} 
+ \frac{4T_0\sqrt{2(8+n)}}{\lambda \eps_1 \eps_2}.
\end{equation}

Then, using the fact that $\frac{||M_{post}||_{2,2}}{\sqrt{T-T_0}} \leq \sqrt{n}$ and setting $\eps := \eps_1 = \eps_2$ and $\lambda = O(T_0)$, Equation \eqref{eq.accu.diff} can be bounded by,
\[
\frac{4T_0\sqrt{(8+n)n}}{\lambda \eps} + 
\frac{4T_0\sqrt{(8+n)n \sigma^2}}{\lambda \eps} +  \frac{\sqrt{2n}\psi}{\eps} 
+ \frac{4T_0\sqrt{2(8+n)}}{\lambda \eps^2 } = O\left(\frac{n}{\eps} + \frac{\sqrt{n}}{\eps^2}\right) = O\left(\frac{n}{\eps}\right) \text{ for }\eps \geq 1/\sqrt{n}.
\]
Thus we conclude that the cost of privacy in the $DPSC_{out}$ algorithm is at most a factor of $O(\frac{1}{\eps})$. The restriction to $\eps \geq \frac{1}{\sqrt{n}}$ is consistent with standard practice in both theoretical and practical deployments of differential privacy, and thus is effectively without loss.

\subsection{Closed-form bound on RMSE of Output Perturbation}\label{s.accdist}

In this section, we impose assumptions on the underlying data distribution to extend Theorem \ref{thm.acc.out} to provide an explicit closed-form bound on the RMSE. Throughout this section, we make the following three mild assumptions of the distribution of $X$, which are required to achieve this closed-form expression:
\begin{assumption}\label{assum.1}
$X_{pre}$ takes values in a $k$-dimensional subspace $E$ for some small $k \ll \min\{n,T_0\}$.
\end{assumption}

\begin{assumption}\label{assum.2}
The distribution of $X_{pre}$ over $E$ is isotropic, hence the covariance matrix $Cov(X_{pre}) = \Sigma = P_E$ where $P_E$ is an orthogonal projection matrix onto $E$.
\end{assumption}

\begin{assumption}\label{assum.3}
The distribution of $\bm{x_t} \in \mathbb{R}^n$ is supported in some centered Euclidean ball with radius $O(\sqrt{k})$.
\end{assumption}

These assumptions are only slightly stronger than those commonly made in theory \cite{rsc, mrsc} and that typically hold in practice \cite{udell2019lowrank}.
The first assumption means that $X_{pre}$ is low rank. Assuming $X_{pre}$ to be \emph{approximately} low rank is a common practice in synthetic control literature \cite{rsc, mrsc}. Indeed, most large matrices in practice are approximately low-rank \cite{udell2019lowrank}. Hence, we only further assume that it is \emph{exactly} rank $k$ for some small $k$. The second assumption allows us to apply useful mathematical properties: $||P_E||_2=1$ and $trace(P_E)=k$. Then, $\mathbb{E}[X_{pre}X_{pre}^{\top}] = trace(\Sigma) = k$ and, using Markov's inequality, we know that most of the distribution mass should be within a ball of radius $\sqrt{m}$ for $m = O(k)$. Hence, the third assumption asserts that not \emph{most} but \emph{all} the probability mass should lie within that ball, i.e., $||X||_{2,2}=O(\sqrt{k})$ almost surely.

Corollary \ref{cor.accuracy} provides a closed-form bound on the RMSE of $\bm{y}^{out}$ under these assumptions.

\begin{corollary}\label{cor.accuracy}
If Assumptions \ref{assum.1}, \ref{assum.2}, and \ref{assum.3} hold, then for all $\xi \in (0,1)$ and $t\geq 1$, with probability at least $1-n^{-t^2}$, if $T_0 \geq  C(t/\xi)^2 k \log n$, we have
\[
    RMSE(\bm{y}^{out})
    \leq \sqrt{n}\left( \frac{ (\sqrt{2n\sigma^2} + \sqrt{2n\sigma^2s^2} \;)T_0 + \frac{\lambda}{2T_0}}{ (1-\xi)T_0 + \frac{\lambda}{2T_0}}+\frac{4T_0\sqrt{8+n}}{\lambda \eps_1} \right) + \left(\sqrt{n \sigma^2} + \frac{\sqrt{2}}{\eps_2} \right) \left( \sqrt{n} \psi + \frac{4T_0\sqrt{8+n}}{\lambda \eps_1}\right).
\]
\end{corollary}

To derive Corollary \ref{cor.accuracy} from Theorem \ref{thm.acc.out}, we only need to derive and apply bounds on $||M_{post}||_{2,2}$ and $\mathbb{E}[||\f^{reg}-\f||_2]$. As before, we bound the first term using 
\[
||M_{post}||_{2,2} \leq ||M_{post}||_F \leq \sqrt{n(T-T_0)}, 
\]
and thus $\frac{||M_{post}||_{2,2}}{\sqrt{T-T_0}} \leq \sqrt{n}$. 
Therefore, the key step is to bound $\mathbb{E}[||\f^{reg}-\f||_2]$. The following lemma provides the required bound on this term to prove Corollary \ref{cor.accuracy}. The remainder of this section will be devoted to providing a proof sketch for Lemma \ref{lem.estbounds}. A full proof is presented in Appendix \ref{app.lemestbounds}.

\begin{restatable}{lemma}{estbounds}\label{lem.estbounds}
Let $\f^{reg} = (X_{pre} X_{pre}^{\top} + \frac{\lambda}{2T_0} I)^{-1} X_{pre}\bm{y}_{pre}$ be the Ridge regression coefficients and let $\f$ be the true coefficients.
If Assumptions \ref{assum.1}, \ref{assum.2}, and \ref{assum.3} hold, then for all $\xi \in(0,1)$ and $t\geq 1$, with probability at least $1-n^{-t^2}$, if $T_0 \geq  C(t/\xi)^2 k \log n$, then,
\[ \mathbb{E}[||\f^{reg}-\f||_2] \leq \frac{ (\sqrt{2n\sigma^2} + \sqrt{2n\sigma^2s^2})T_0 + \frac{\lambda}{2T_0}}{ (1-\xi)T_0 +\frac{\lambda}{2T_0}}.
\]
\end{restatable}

\begin{proof}[Proof sketch of Lemma \ref{lem.estbounds}.]
First we can expand $\mathbb{E}[||\f^{reg}-\f||_2]$:
\begin{align}\label{eq.f.l2body}
    \mathbb{E}[||\f^{reg}-\f||_2] &=  \mathbb{E}[||\f^{reg} -\mathbb{E}[\f^{reg}] +\mathbb{E}[\f^{reg}] - \f||_2] \notag \\
    &\leq \mathbb{E}[||\f^{reg} -\mathbb{E}[\f^{reg}] ||_2] + \mathbb{E}[||\mathbb{E}[\f^{reg}] - \f||_2] \notag\\
    &= \mathbb{E}[||\f^{reg} -\mathbb{E}[\f^{reg}] ||_2] + \mathbb{E}[||Bias(\f^{reg})||_2].
\end{align}
We can bound these two terms separately as:
\begin{align*}
Bias(\f^{reg}) & \leq \frac{\lambda}{2T_0} ||(X_{pre} X_{pre}^{\top} +\frac{\lambda}{2T_0} I)^{-1}||_{2,2} \quad \text{ and}\\
\mathbb{E}[|| \f^{reg} - \mathbb{E}[\f^{reg}] ||_2] &\leq \mathbb{E}[|| (X_{pre} X_{pre}^{\top} +\frac{\lambda}{2T_0} I)^{-1}||_{2,2} \cdot ||X_{pre}\bm{z}_{pre} -X_{pre}Z_{pre}^{\top}\f ||_2].
\end{align*}
This can be combined back with Equation \eqref{eq.f.l2body} to yield,
\begin{equation}\label{eq.estbound.interim.1body}
    \mathbb{E}[||\f^{reg}-\f||_2] \leq \mathbb{E}[|| (X_{pre} X_{pre}^{\top} +\frac{\lambda}{2T_0} I)^{-1}||_{2,2} \cdot ( ||X_{pre}\bm{z}_{pre} -X_{pre}Z_{pre}^{\top}\f ||_2+\frac{\lambda}{2T_0} )].
\end{equation}
Next, we use our assumptions on the data distribution to prove the following lemma about $|| (X_{pre} X_{pre}^{\top} +\frac{\lambda}{2T_0} I)^{-1}||_{2,2}$.

\begin{restatable}{lemma}{eigenlemma}
\label{lem.eigen}
If Assumptions \ref{assum.1}, \ref{assum.2}, and \ref{assum.3} hold, then for all $\xi \in (0,1)$ and $t\geq 1$, with probability at least $1-n^{-t^2}$, if $T_0 \geq  C(t/\xi)^2 k \log n$, then,
\[
||(X_{pre}X_{pre}^{\top} +\frac{\lambda}{2T_0} I)^{-1} ||_{2,2} \leq \frac{1}{ (1-\xi)T_0 +\frac{\lambda}{2T_0}}.
\]
\end{restatable}

To prove Lemma \ref{lem.eigen}, we use the following lemma about concentration of random matrices.

\begin{restatable}[Corollary 5.52 of \cite{vershynin2010introduction}]{lemma}{corollarylemma}
\label{lem.cor552}
Consider a distribution in $\mathbb{R}^n$ with covariance matrix $\Sigma$, and supported in some centered Euclidean ball whose radius we denote $\sqrt{m}$. Let $T_0$ be the number of samples and define the sample covariance matrix $\Sigma_{T_0} = \frac{1}{T_0}XX^{\top}$. Let $\xi \in (0,1)$ and $t\geq 1$. Then with probability at least $1-n^{-t^2}$, 
\[
\text{If} \quad T_0\geq C(t/\xi)^2 ||\Sigma||_{2,2}^{-1} m \log n \quad \text{then} \; ||\Sigma_{T_0} - \Sigma||_{2,2}\leq \xi ||\Sigma||_{2,2},
\]
where $C$ is an absolute constant.
\end{restatable}

We instantiate Lemma \ref{lem.cor552} using our assumptions that $||\Sigma||_{2,2} =||P_E||_{2,2}=1$ and the distribution is supported within some centered Euclidean ball with radius $\sqrt{O(k)}$ to get that 
with probability at least $1-n^{-t^2}$ and $T_0 \geq  C(t/\xi)^2  k \log n$, 
\[ ||\frac{1}{T_0}X_{pre}X_{pre}^{\top} -\Sigma||_{2,2} \leq \xi.\]
We then use this to show that
\[|| X_{pre}X_{pre}^{\top} - T_0 I ||_{2,2} \leq \xi T_0,\]
and thus all eigenvalues of $(X_{pre}X_{pre}^{\top} -  T_0 I )$ must be at most $\xi T_0$, so 
\[ (1- \xi) T_0  \leq \bm{\lambda_{min}}(X_{pre}X_{pre}^{\top}) \leq  (1+ \xi) T_0.\]

Finally, we can complete the proof of Lemma \ref{lem.eigen}, by observing that,
\[ || (X_{pre}X_{pre}^{\top} + \frac{\lambda}{2T_0} I)^{-1} ||_{2,2} = \frac{1}{|\bm{\lambda_{min}}(X^{\top}X) +\frac{\lambda}{2T_0}|} \leq \frac{1}{ (1-\xi)T_0 +\frac{\lambda}{2T_0}}.
\]
Returning to Equation \eqref{eq.estbound.interim.1body}, we can use this bound---along with the model properties specified in Equation \eqref{eq.noise} that each element of $\bm{z}$ and $Z$ has mean $0$, variance $\sigma^2$, and support $[-s,s]$---to obtain the desired bound:
\begin{align*}
    \mathbb{E}[||\f^{reg}-\f||_2] 
    &\leq \frac{1}{ (1-\xi)T_0 +\frac{\lambda}{2T_0}} \mathbb{E}[ ||X_{pre}\bm{z}_{pre} -X_{pre}Z_{pre}^{\top}\f ||_2+\frac{\lambda}{2T_0} ] \\
    &\leq \frac{1}{ (1-\xi)T_0 +\frac{\lambda}{2T_0}} \left((\sqrt{n T_0}  + \sqrt{n T_0 s^2})\mathbb{E}[||\bm{z}_{pre} - Z_{pre}^{\top}\f||_2] +\frac{\lambda}{2T_0}\right) \\
    &\leq \frac{1}{ (1-\xi)T_0 +\frac{\lambda}{2T_0}} \left((\sqrt{n T_0}  + \sqrt{n T_0 s^2})\sqrt{2T_0 \sigma^2} +\frac{\lambda}{2T_0} \right)\\   
    &\leq \frac{ (\sqrt{2n\sigma^2} + \sqrt{2n\sigma^2s^2})T_0 + \frac{\lambda}{2T_0}}{ (1-\xi)T_0 +\frac{\lambda}{2T_0}}.
\end{align*}
\end{proof}

\section{Privacy Guarantees of $DPSC_{obj}$ }\label{s.privacyobj}

In this section, we prove Theorem \ref{thm.priv.obj}, that $DPSC_{obj}$ is $(\eps_1 + \eps_2,\delta)$-differentially private. This proof relies on composition of the $(\eps_1,\delta)$-DP learning step and the $(\eps_2,0)$-DP prediction step. The prediction step is identical to that of Algorithm \ref{alg.output}, so the privacy of this step follows immediately from Lemma \ref{step2.privacy} (that $\tilde{X}_{post}$ is computed in an $(\eps_2,0)$-DP manner) and post-processing on the DP output of Step 1. All the remains to be shown is that $\bm{f}^{obj}$ is computed in an $(\eps_1,\delta)$-DP manner (Theorem \ref{step1.privacyobj}), and then Theorem \ref{thm.priv.obj} will follow by basic composition.

\begin{theorem}\label{step1.privacyobj}
Step 1 of Algorithm \ref{alg.obj} that computes $\f^{obj}$
is $(\eps_1,\delta)$-differentially private.
\end{theorem}

At a high-level, the privacy of $\f^{obj}$ comes from a carefully modified instantiation of the Objective Perturbation algorithms of \cite{CMS11, KST12}, with novel sensitivity analysis, again due to the transposed regression setting of synthetic control (i.e., along columns not rows), where privacy is still required along the rows.

More formally, we start with the standard Ridge Regression objective function $J(\f)$, that can be separated into the MSE loss function $\cL(\f)$ and the regularization term $r(\f)=\frac{\lambda}{2T_0} ||\f||_2^2$ as follows:
\[
    J(\f) = \cL(\f) + r(\f) = \frac{1}{T_0} ||y_{pre} - {X}_{pre}^{\top} \f||_2^2 + \frac{\lambda}{2T_0} ||\f||_2^2.
\]
The Objective Perturbation method modifies $J(\f)$ by adding two terms: an additional regularization term and a noise term to ensure privacy: 
\[ J^{obj}(\f) = J(\f) + \frac{\Delta}{2T_0} ||\f||_2^2 + \frac{1}{T_0}\bm{b}^{\top}\f
= \cL(\f) + \frac{\lambda + \Delta}{2T_0} ||\f||_2^2 + \frac{1}{T_0}\bm{b}^{\top}\f,
\]
where $\bm{b}$ is a random vector drawn from a high-dimensional Laplace distribution if $\delta=0$, and from a multivariate Gaussian distribution if $\delta>0$.

Notice that $J^{obj}(\f)$ is strongly convex (for any $\Delta \geq 0$) and differentiable. Hence, for any given input dataset $\D=(X_{pre}, y_{pre})$ and any fixed parameters ($\lambda, \eps_1, \eps_2, \delta$), there exists a bijection between a realized value of the noise term $\bm{b}$ and $\f^{obj}:= \argmin_{\f} J^{obj}(\f)$ given that realized $\bm{b}$.\footnote{For a simple analogy, consider the one-dimensional Laplace Mechanism on query $f$ and database $x$, which outputs $y=f(x) + Lap(\Delta f / \eps)$. Given $f$ and $x$, there is a bijection between noise terms and outputs since the noise term must equal $y-f(x)$.} We can then use this bijection to analyze the distribution over outputs on neighboring databases via the (explicitly given) noise distribution.

To observe this bijection concretely, let $\bm{b}(\bm{\alpha};\D)$ be noise value that must have been realized when database $\D$ was input and $\bm{\alpha} =\arg\min_{\f} J^{obj}(\f)$ was output.
We can derive a closed-form expression for $\bm{b}(\bm{\alpha};\D)$ by computing the gradient of $J^{obj}(\f)$, which should be zero when evaluated at $\f = \bm{\alpha}$ since $\bm{\alpha}$ is defined to be the minimizer of $J^{obj}(\f)$: 
\[\nabla J^{obj}(\f) \big|_{\f = \bm{\alpha}} = \nabla\cL(\bm{\alpha}) +
\nabla r(\bm{\alpha})+
\frac{\Delta}{T_0}\bm{\alpha}
+ \frac{\bm{b}(\bm{\alpha};\D)}{T_0} \stackrel{!}{=} 0.\]
Rearranging the equation yields
\[
\bm{b}(\bm{\alpha};\D) = -\left( T_0 \nabla\cL(\bm{\alpha}) + T_0 \nabla r(\bm{\alpha})
+ \Delta\bm{\alpha} \right).
\]

Now, consider two arbitrary neighboring databases $\D$ and $\D'$ and an arbitrary output value $\bm{\alpha}$. Similar to \cite{CMS11}, we can use, e.g., \cite{billingsley1995measure} to express the ratio of the probabilities of outputting $\bm{\alpha}$ on neighboring $\D$ and $\D'$ as:\footnote{with abuse of notation to let $\Pr$ denote pdf for simplicity of presentation.}
\[
\frac{\Pr(\f^{obj}=\bm{\alpha}\;|\;\D)}{\Pr(\f^{obj}=\bm{\alpha}\;|\;\D')} =  \frac{\Pr(\bm{b}(\bm{\alpha};\D))}{\Pr(\bm{b}(\bm{\alpha};\D'))} \frac{|det(\nabla\bm{b}(\bm{\alpha};\D'))|}{|det(\nabla\bm{b}(\bm{\alpha};\D))|}
:= \Gamma(\bm{\alpha}) \cdot \Phi(\bm{\alpha};\Delta), 
\]
where we define $\Gamma(\bm{\alpha}) := \frac{\Pr(\bm{b}(\bm{\alpha};\D))}{\Pr(\bm{b}(\bm{\alpha};\D'))}$ and 
$\Phi(\bm{\alpha};\Delta) := \frac{|det(\nabla\bm{b}(\bm{\alpha};\D'))|}{|det(\nabla\bm{b}(\bm{\alpha};\D))|}$.
In the remainder of the proof, we will bound $\Gamma(\bm{\alpha}) \leq e^{\eps_0}$ and $\Phi(\bm{\alpha};\Delta) \leq e^{\eps_1 -\eps_0}$ so that the product is bounded by $e^{\eps_1}$.

The parameter $\Delta$ serves a role to divide the $\eps_1$ budget between these two terms, by distinguishing between two cases. In the first case, $\eps_1$ is large enough that we can choose $\Delta=0$ and still have some privacy budget ($\eps_0$) remaining to bound $\Gamma(\bm{\alpha})$. In the other case, if $\eps_1$ is too small to bound $\Phi(\bm{\alpha};\Delta)$ with $\Delta=0$, then we divide the privacy budget equally between bounding $\Gamma(\bm{\alpha})$ and $\Phi(\bm{\alpha};\Delta)$, and find an appropriate value for $\Delta>0$.

First, we will show $\Phi(\bm{\alpha};\Delta)$ is upper bounded by $e^{\eps_1 - \eps_0}$. 

\begin{lemma}\label{lem.phi}
If $\Delta=0$ and $\eps_0=\eps_1-\log(1+\frac{2c}{\lambda}+\frac{c^2}{\lambda^2})$, or if $\Delta=\frac{c}{e^{\eps_1/4} -1 }-\lambda$ and $\eps_0 = \frac{\eps_1}{2}$, then
$\Phi(\bm{\alpha};\Delta) \leq e^{\eps_1 - \eps_0}$.
\end{lemma}

\begin{proof}
We start with Lemma \ref{lem.step2} (proved in Appendix \ref{app.lemstep2}), which bounds $\Phi(\bm{\alpha};\Delta)$ as a function of $\lambda$, $c$, and $\Delta$.

\begin{restatable}{lemma}{steptwo}\label{lem.step2}
 For any $\Delta \geq 0$, $\Phi(\bm{\alpha};\Delta) = \frac{|det(\nabla\bm{b}(\bm{\alpha};\D'))|}{|det(\nabla\bm{b}(\bm{\alpha};\D))|} \leq (1 + \frac{c}{\lambda+\Delta})^2$.
\end{restatable}

Next, we use this result to prove our desired bound that $\Phi(\bm{\alpha};\Delta) \leq e^{\eps_1 -\eps_0}$. We do this by considering two cases. First, when $\Delta = 0$, then $\Phi(\bm{\alpha};\Delta=0) \leq 1+\frac{2c}{\lambda}+\frac{c^2}{\lambda^2} \leq e^{\eps_1 - \eps_0}$ by design, where the first inequality comes from Lemma \ref{lem.step2} and the second inequality comes from the choice of $\eps_0 = \eps_1 - \log(1+\frac{2c}{\lambda} + \frac{c^2}{\lambda^2})$ when $\Delta=0$. In the second case, $\Delta = \frac{c}{e^{\eps_1/4}-1}-\lambda$. Plugging this $\Delta$ value into the bound of Lemma \ref{lem.step2} gives $\Phi(\bm{\alpha};\Delta) \leq e^{\eps_1/2} = e^{\eps_1 - \eps_0}$, where the second inequality come from our choice of $\eps_0 = \eps_1/2$.
\end{proof}

Next, we bound $\Gamma(\bm{\alpha}) = \frac{\Pr(\bm{b}(\bm{\alpha};\D))}{\Pr(\bm{b}(\bm{\alpha};\D'))}$. Note that this term depends only on the noise distribution, and not on the value of $\Delta$. Algorithm \ref{alg.obj} offers two options of noise distributions: Laplace noise when $\delta=0$, and Gaussian noise when $\delta>0$. 

In the case of Laplace noise, the bound that $\Gamma(\bm{\alpha}) \leq e^{\eps_0}$ follows immediately from the Laplace mechanism instantiated with privacy parameter $\eps_0$ and Lemma \ref{lem.goff} to bound the sensitivity. The following lemma is proved in Appendix \ref{app.lemlaplace}.

\begin{restatable}{lemma}{laplacenoise}\label{lem.laplacenoise}
When $\bm{b}$ is sampled according to pdf $p(\bm{b};\beta) \propto \exp \left( -\frac{ ||\bm{b}||_2}{\beta} \right)$, where $\beta = \min\{ \frac{4 T_0 \sqrt{8+n}}{\eps_0}, \frac{c\sqrt{n}+4T_0}{\eps_0}\}$, then $\Gamma(\bm{\alpha}) = \frac{\Pr(\bm{b}(\bm{\alpha};\D))}{\Pr(\bm{b}(\bm{\alpha};\D'))} \leq e^{\eps_0}$.
\end{restatable}

The two different $\beta$ values come from two different upper bounds on the sensitivity, and the minimum value will give a tighter bound.

In the case where $\delta>0$ and the Gaussian Mechanism is used, we cannot simply bound $\Gamma(\bm{\alpha}) = \frac{\Pr(\bm{b}(\bm{\alpha};\D))}{\Pr(\bm{b}(\bm{\alpha};\D'))}$ with probability 1. Instead, the bound must incorporate the $\delta$ term to bound $\Gamma(\bm{\alpha})$ with probability $1-\delta$ over the internal randomness of the algorithm, as in Lemma \ref{lem.step4gauss}, formally proven in Appendix \ref{app.lemstep4gauss}.

\begin{restatable}{lemma}{stepfourgauss}\label{lem.step4gauss}
When $\bm{b} \sim \mathcal{N}(0, \beta^2 I_n)$, where $\beta = \frac{4 T_0 \sqrt{8+n} \sqrt{2 \log \frac{2}{\delta} + \eps_0}}{\eps_0}$, then $\Gamma(\bm{\alpha}) = \frac{\Pr(\bm{b}(\bm{\alpha};\D))}{\Pr(\bm{b}(\bm{\alpha};\D'))} \leq e^{\eps_0}$ with probability at least $1-\delta$.
\end{restatable}

Finally, we combine the bounds on $\Phi(\bm{\alpha};\Delta)$ and $\Gamma(\bm{\alpha})$ to complete the proof. When $\delta=0$ with Laplace noise, Lemmas \ref{lem.phi} and \ref{lem.laplacenoise} combine immediately to give that $\Phi(\bm{\alpha};\Delta) \Gamma(\bm{\alpha}) \leq e^{\eps_1 - \eps_0 + \eps_0} = e^{\eps_1}$. When $\delta>0$ and Gaussian noise is used, we define $\mathcal{G}$ to be the good event that $\Gamma(\bm{\alpha}) \leq e^{\eps_0}$, which we know from Lemma \ref{lem.step4gauss} will happen with at least probability $1-\delta$. Then conditioned on $\mathcal{G}$ we have,
\[
    \frac{\Pr(\f^{obj}=\bm{\alpha}\;|\; \D, \mathcal{G}) }{\Pr(\f^{obj}=\bm{\alpha}\;| \; \D', \mathcal{G})} = \Gamma(\bm{\alpha}) \cdot \Phi(\bm{\alpha};\Delta) \leq e^{\eps_0} \cdot e^{\eps_1 - \eps_0} \leq e^{\eps_1}.
\]
We can then use this fact to derive our desired (unconditioned) privacy bound:
\begin{equation*}
\begin{split}
    \Pr(\f^{obj} = \bm{\alpha} \;|\; \D) &= \Pr(\mathcal{G}) \cdot \Pr( \f^{obj}=\bm{\alpha}\; |\; \D, \mathcal{G})+  \Pr(\overline{\mathcal{G}}) \cdot \Pr( \f^{obj}=\bm{\alpha}\; |\; \D, \overline{\mathcal{G}})\\
    &\leq e^{\eps_1} \Pr(\mathcal{G}) \cdot \Pr( \f^{obj}=\bm{\alpha}\; | \; \D', \mathcal{G}) + \delta\\
    &\leq e^{\eps_1} \Pr( \f^{obj}=\bm{\alpha}\;|\;\D') + \delta.
\end{split}
\end{equation*}

Hence, $\f^{obj}$ in Algorithm \ref{alg.obj} is $(\eps_1, \delta)$-DP and the final output $\bm{y}^{obj}$ is $(\eps_1 + \eps_2, \delta)$-DP by composition.

\section{Accuracy Guarantees of $DPSC_{obj}$}\label{s.accuracyobj}

In this section we analyze the accuracy of $DPSC_{obj}$. We first prove Theorem \ref{thm.acc.obj}, restated below for convenience.

\mainaccobj*

This theorem gives bounds on the predicted post-intervention target vector $\bm{y}^{obj}$, as measured by RMSE.
Similar to Theorem \ref{thm.acc.out}, this result is stated in full generality with respect to the distribution of data and the latent variables, and thus the bound depends on terms such as $||M_{post}||_{2,2}$ and $\mathbb{E}[||\f^{reg}-\f||_2]$.
Section \ref{s.accpostobj} provides a proof of this result, with omitted detailed deferred to Appendix \ref{app.proofsobj}.

Comparing the bound of Theorem \ref{thm.acc.obj} to that of Theorem \ref{thm.acc.out} for output perturbation, we see that the difference comes only from the respective terms $\E[||(\f^{(out \lor obj)}-\f^{reg}) ||_2]$. For output perturbation, the error $\f^{out}-\f^{reg}$ is simply the noise directly added to the output, so the expected norm of the error is simply the expected norm of the noise,  $a=\frac{4T_0\sqrt{8+n}}{\lambda \eps_1}$. For objective perturbation, the interpretation of this error terms is less straightforward, and is instead bounded using Lemma \ref{lem.fobjacc}.
As a simple case for comparison, when $\Delta=0$ and $\delta=0$ (i.e., using Laplace noise), the expected difference becomes $\E[||(\f^{obj}-\f^{reg}) ||_2] \leq \min\{ \frac{8T_0^2\sqrt{8+n}}{\lambda \eps_0}, \frac{2cT_0\sqrt{n}}{\lambda \eps_0}\}$. If the first term is the smaller of the two, then  $\E[||(\f^{obj}-\f^{reg}) ||_2]$ is bigger than $\E[||(\f^{out}-\f^{reg}) ||_2]$ because of the additional $T_0$ factor, and since $\eps_0<\eps_1$ (assuming the same $\eps_1$ values for comparison). If the second term is the minimum, then the upper bound on error is larger under objective perturbation when $c \geq \frac{2\sqrt{8+n}}{\sqrt{n}}$. In both cases, this diverges from the conclusions when comparing these methods for the standard ERM setting, where objective perturbation is known to provide better performance than output perturbation \cite{CMS11}.

\subsection{Accuracy of post-intervention prediction via objective perturbation $\emph{y}^{obj}$}\label{s.accpostobj}

We will prove Theorem \ref{thm.acc.obj}, which upper bounds the Root Mean Squared Error (RMSE) of $\bm{y}^{obj}$, defined as:
\begin{equation*}
RMSE(\bm{y}^{obj}) = \frac{1}{\sqrt{T-T_0}}\E[||\bm{y}^{obj}-\bm{m_{post}}||_2].
\end{equation*}

Using the facts that $\bm{y}^{obj} = \tilde{X}_{post}^{\top} \bm{f}^{obj}$, $\tilde{X}_{post} = X_{post} + M_{post}+Z_{post}$, and $\bm{m}_{post} = M_{post}^{\top}\f$ (by Equation \eqref{eq.model}), we can bound the expectation as follows:
\begin{align}\label{eq.firstbound}
    \E[||\bm{y}^{obj}-\bm{m}_{post}||_2]
    &= \E[|| \tilde{X}_{post}^\top \f^{obj} - M_{post}^\top \f ||_2] \notag\\
    &= \E[|| \tilde{X}_{post}^\top \f^{obj} - \tilde{X}_{post}^\top \f^{reg} + \tilde{X}_{post}^\top \f^{reg} - M_{post}^\top \f ||_2] \notag\\
    &= \E[|| (M_{post}+Z_{post}+W_{post})^\top (\f^{obj} - \f^{reg}) + (M_{post}+Z_{post}+W_{post})^\top \f^{reg} - M_{post}^\top \f ||_2] \notag\\
    &\leq \E[|| (M_{post}+Z_{post}+W_{post})^\top (\f^{obj}-\f^{reg}) ||_2] \notag\\
    &\quad + \E[||M_{post}^\top (\f^{reg} - \f) ||_2] + \E[|| (Z_{post}+W_{post})^\top \f^{reg} ||_2],
\end{align}
where the first equality is due to the definition of $\bm{y}^{obj}$, the second equality adds and subtracts the same term, the third equality collects terms and plugs in the expression for $\tilde{X}_{post}$, and the final step is due to triangle inequality.

Lemma \ref{lem.threebounds} already bounds the last two terms because they do not involve $\f^{obj}$ and Step 2 of Algorithms \ref{alg.output} and \ref{alg.obj} are the same. Specifically, we know that,
\begin{equation}\label{eq.secondbound}
\mathbb{E}[|| M_{post}^{\top}(\f^{reg} -\f)||_2]
\leq ||M_{post}||_{2,2} \mathbb{E}[||\f^{reg} -\f||_2] \; \text{and} \; \mathbb{E}[||(Z_{post}^{\top} + W_{post}^{\top}) \f^{reg} ||_2]
\leq \sqrt{n} \psi ( \sqrt{n(T-T_0) \sigma^2} + \tfrac{2\sqrt{T-T_0}}{\eps_2}  ).
\end{equation}

Thus we only need to bound the first term: 
\begin{align}\label{eq.thirdbound}
    \E[|| (M_{post}&+Z_{post}+W_{post})^\top (\f^{obj}-\f^{reg}) ||_2] \notag \\
    &\leq \E[|| M_{post}^\top (\f^{obj}-\f^{reg}) ||_2] + \E[|| (Z_{post}+W_{post})^\top (\f^{obj}-\f^{reg}) ||_2] \notag \\
    &\leq || M_{post}||_2 \E[||\f^{obj}-\f^{reg} ||_2] + \E[||Z_{post}+W_{post}||_2] \E[||(\f^{obj}-\f^{reg}) ||_2] \notag \\
     &\leq || M_{post}||_2 \E[||\f^{obj}-\f^{reg} ||_2] + ( \sqrt{n(T-T_0)\sigma^2} + \frac{2\sqrt{T-T_0}}{\eps_2} ) \E[||(\f^{obj}-\f^{reg}) ||_2],
\end{align}
where the first step is simply triangle inequality, the second step is due to the independence of $Z, W$ and $\f^{obj}$, and the third step comes from the proof of Lemma \ref{lem.threebounds} (see Appendix \ref{app.proofthreebounds}), where $\E[||Z_{post}+W_{post}||_2]$ was bounded as an intermediate step.

Thus we only need to derive a bound on $\E[||\f^{obj}-\f^{reg} ||_2]$, which we do in Lemma \ref{lem.fobjacc} (formally proven in Appendix \ref{app.lemfobjacc}) to complete the proof.

\begin{restatable}{lemma}{fobjacc}\label{lem.fobjacc}
The $\ell_2$ distance between $\f^{obj}$ and $\f^{reg}$ satisfies:
\[
\E[||\f^{obj}-\f^{reg} ||_2] \leq  \frac{2}{\lambda+\Delta}\E[||\bm{b}||_2] + \mathds{1}_{\Delta \neq 0} \left( \frac{1}{\lambda} + \frac{1}{\lambda+\Delta} \right) 2T_0^2 \sqrt{n},
\]
where $\bm{b}$ and $\Delta$ are computed internally by Algorithm \ref{alg.obj}.
\end{restatable}

Combining Equations \eqref{eq.firstbound}, \eqref{eq.secondbound}, and \eqref{eq.thirdbound} with Lemma \ref{lem.fobjacc} completes the proof of Theorem \ref{thm.acc.obj}:
\begin{align*}
RMSE(\bm{y}^{obj})
&\leq \frac{||M_{post}||_2}{\sqrt{T-T_0}} \left( \E[||(\f^{reg}-\f) ||_2] + \E[||(\f^{obj}-\f^{reg}) ||_2] \right)\\
&\quad + \left( \sqrt{n \sigma^2} + \frac{\sqrt{2}}{\eps_2} \right) \left( \sqrt{n} \psi + \E[||(\f^{obj}-\f^{reg}) ||_2] \right)\\
&\leq \frac{||M_{post}||_2}{\sqrt{T-T_0}} \left( \E[||(\f^{reg}-\f) ||_2] + \frac{2}{\lambda+\Delta}\E[||\bm{b}||_2] + \mathds{1}_{\Delta \neq 0} \left( \frac{1}{\lambda} + \frac{1}{\lambda+\Delta} \right) 2T_0^2 \sqrt{n} \right)\\
&\quad + \left( \sqrt{n \sigma^2} + \frac{\sqrt{2}}{\eps_2} \right) \left( \sqrt{n} \psi + \frac{2}{\lambda+\Delta}\E[||\bm{b}||_2] + \mathds{1}_{\Delta \neq 0} \left( \frac{1}{\lambda} + \frac{1}{\lambda+\Delta} \right) 2T_0^2 \sqrt{n} \right),
\end{align*}
where $\E[||\bm{b}||_2] = \sqrt{n\beta} = \sqrt{ \frac{n T_0 \zeta \sqrt{2 \log \frac{2}{\delta} + \eps_0}}{\eps_0} }$ for Gaussian noise and 
$\E[||\bm{b}||_2] = \min\{ \frac{4 T_0 \sqrt{8+n}}{\eps_0}, \frac{c\sqrt{n}+4T_0}{\eps_0}\}$ for Laplace noise.

\subsection{Closed-form bound on RMSE of Objective Perturbation}\label{s.objclosedform}

Using similar analysis as in Section \ref{s.accdist}, we can extend Theorem \ref{thm.acc.obj} to obtain the following closed-form accuracy bound that depends only on explicit input parameters, under the same distributional assumptions.

\begin{corollary}\label{cor.accuracy.obj}
If Assumptions \ref{assum.1}, \ref{assum.2}, and \ref{assum.3} hold, then for all $\xi \in (0,1)$ and $t\geq 1$, with probability at least $1-n^{-t^2}$, if $T_0 \geq  C(t/\xi)^2 k \log n$, we have
\begin{align*}
RMSE(\bm{y}^{out})
&\leq \sqrt{n}\left( \frac{ (\sqrt{2n\sigma^2} + \sqrt{2n\sigma^2s^2} \;)T_0 + \frac{\lambda}{2T_0}}{ (1-\xi)T_0 + \frac{\lambda}{2T_0}}+ \frac{2}{\lambda+\Delta}\E[||\bm{b}||_2] + \mathds{1}_{\Delta \neq 0} \left( \frac{1}{\lambda} + \frac{1}{\lambda+\Delta} \right) 2T_0^2 \sqrt{n} \right)\\
&\quad + \left(\sqrt{n \sigma^2} + \frac{\sqrt{2}}{\eps_2} \right) \left( \sqrt{n} \psi + \frac{2}{\lambda+\Delta}\E[||\bm{b}||_2] + \mathds{1}_{\Delta \neq 0} \left( \frac{1}{\lambda} + \frac{1}{\lambda+\Delta} \right) 2T_0^2 \sqrt{n} . \right),
\end{align*}
where $||\f^{reg}||_{\infty} \leq \psi$ for some $\psi>0$, and $\E[||\bm{b}||_2] = \sqrt{ \frac{n T_0 4\sqrt{8+n} \sqrt{2 \log \frac{2}{\delta} + \eps_0}}{\eps_0} }$ for Gaussian noise ($\delta>0$ case) and 
$\E[||\bm{b}||_2] = \min\{ \frac{4 T_0 \sqrt{8+n}}{\eps_0}, \frac{c\sqrt{n}+4T_0}{\eps_0}\}$ for Laplace noise ($\delta=0$ case), and $\eps_0$ and $\Delta$ are computed internally by the algorithm.
\end{corollary}

The additional terms that arise due to the noise required to guarantee differential privacy in this setting, relative to the bound on $RMSE(\bm{y}^{reg})$ in Equation \eqref{eq.rmse.yreg}, are:
\begin{equation}
\label{eq.obj.closed.bound}
    (\sqrt{n}+\sqrt{n\sigma^2}+\frac{\sqrt{2}}{\eps_2}) \left( \frac{2}{\lambda+\Delta}\E[||\bm{b}||_2] + \mathds{1}_{\Delta \neq 0} \left( \frac{1}{\lambda} + \frac{1}{\lambda+\Delta} \right) 2T_0^2 \sqrt{n} \right) + \frac{\sqrt{2n}}{\eps_2}\psi.
\end{equation}

To analyze this expression in a simplified way, assume the regularization parameter is $\lambda=O(T_0)$ so $\frac{T_0}{\lambda} = O(1)$, and that Laplace noise was used (i.e., $\delta=0$), so that $\E[||\bm{b}||_2] = O(\frac{T_0 \sqrt{n}}{\eps_0})$. Then the first parenthesis of Equation \eqref{eq.obj.closed.bound} becomes $O(\sqrt{n}+\frac{1}{\eps_2})$, the second parenthesis becomes $O(\frac{T_0\sqrt{n}}{\eps_0} + T_0\sqrt{n})$, and the additive term becomes $O(\frac{\sqrt{n}}{\eps_2})$. Since $\eps_0 < \eps_1$, we replace $\eps_0$ by $\eps_1$ in the bounds. Then Equation \eqref{eq.obj.closed.bound} is $O(\frac{T_0 n}{\eps_1} + \frac{T_0 \sqrt{n}}{\eps_1 \eps_2})$ from the product of two parentheses, and omitting the additive term, which is asymptotically dominated by the others.

Comparing to the cost of privacy in Output Perturbation in Corollary \ref{cor.accuracy}, we see that the bound in Corollary \ref{cor.accuracy} does not depend on $T_0$.
This additional dependence on $T_0$ arises for Objective Perturbation from the second parenthesis containing $\E[||\bm{b}||_2 ]$ and the indicator function, which is absent in the output perturbation case.

\section{Experimental Results and Guidance for Parameter Tuning}\label{s.experiments}

In this section, we present experimental results on the empirical performance of both $DPSC_{out}$ and $DPSC_{obj}$. As a baseline, we compare against both our theoretical bounds in Theorems \ref{thm.acc.out} and \ref{thm.acc.obj}, and the empirical performance of the non-private synthetic control (Algorithm \ref{alg.synth}) with a quadratic loss function. In our experiments, we vary the regularization parameter $\lambda$, the number of donors $n$, the number of pre-intervention observations $T_0$, and the privacy parameter $\epsilon$. For a fair comparison, we use $\delta=0$ for the objective perturbation, except in Section \ref{s.exp.eps} where the impact of $\delta$ is explored.

We observe that $DPSC_{obj}$ outperforms $DPSC_{out}$ for many but not all of the parameter regimes considered, and that both methods provide better performance than their theoretical bounds suggests. We also observe that choosing $\lambda=T_0$ yields optimal or near-optimal RMSE across as variety of parameter regimes, and that empirical performance of both private methods improves as $\eps$ grows large.

\subsection{Dataset and Performance Evaluation}

We use synthetic datasets in our experiments, which enables us to observe the impact of varying the relevant parameters in the data and to match the modeling assumptions of Section \ref{s.model}. We use $T_0\in\{10,100\}$ and $n\in\{10,100\}$ when generating the datasets, corresponding to both smaller and larger number of donors and observations, and we always use $T=T_0 + 3$, meaning that the synthetic control algorithm must predict the next three data points in the donor pool.

The true signals $M$ and $\mathbf{m}$ are generated according to a linear model with random slope, formalized as:
\[ M_{i,t} = \theta_i t \quad \mbox{and} \quad m_{t} = \theta_0  t, \; \forall i \in [n], t \in [T], \]
where the $\theta_i$ are sampled i.i.d.~from a truncated Gaussian with mean $4$, variance $1$, and support $[3,5]$.
Elements of the noise terms $Z$ and $\mathbf{z}$ are sampled i.i.d.~from a truncated Gaussian with mean zero, variance $0.1$ and support $[-1,1]$. Following Equation \eqref{eq.noise}, the donor and target data were respectively defined as $X = M+Z$ and $\mathbf{y}=\mathbf{m}+\mathbf{z}$.
Figure \ref{fig.exp.data} shows an example synthetic dataset generated in this way, with the donor data in grey and the target in red.

\begin{figure}[h]
\centering
\includegraphics[width=0.5\textwidth]{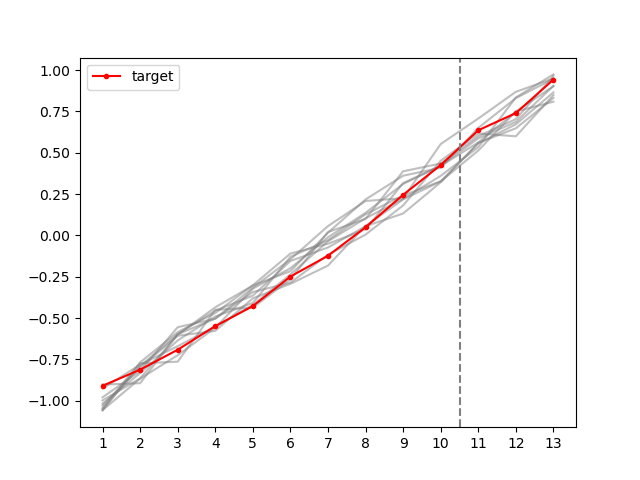}
\caption{Illustration of example synthetic dataset generated with $T_0=10$ and $n=10$. The target time series is in red, and the donor time series are all in grey.}
\label{fig.exp.data}
\end{figure}

In each experiment for a fixed $T_0$ and $n$, a single database was generated and then all algorithms were run 500 times on each dataset. We evaluate post-intervention RMSE as the accuracy measure of interest, as in our theoretical results. Error bands in all figures show $95\%$ confidence intervals, taken over the randomness in data generation and the algorithms.

\subsection{Optimizing regularization parameter $\lambda$}
\label{s.exp.lambda}

The first question we aim to address in our experiments is the impact of the parameter $\lambda$ on performance, and guidance for analysts in their choice of optimal $\lambda$. In our first set of experiments, we fixed $\eps_1=\eps_2=50$, $T_0=10$, and $n=10$ -- other values of $\eps$ and $(T_0,n)$ are considered respectively in Sections \ref{s.exp.eps} and \ref{s.exp.nandT0} -- and empirically measured pre- and post-intervention RMSE as a function of $\lambda$.

Figure \ref{fig.exp.lambda.1} shows the post-intervention RMSE of $DPSC_{out}$, $DPSC_{obj}$, and non-private synthetic control as a function of $\lambda$, for values $\lambda \in \{ 1, 2, 5, 10, 20, 50, 100, 200, 500, 1000, 2000, 5000 \}$. We observe that the performance of the three methods converges as $\lambda$ grows large, but that the RMSE of the private methods are larger than that of the non-private method for smaller $\lambda \leq 20$, with Objective Perturbation substantially outperforming Output Perturbation.

\begin{figure}[t]
\centering
\includegraphics[width=0.75\textwidth]{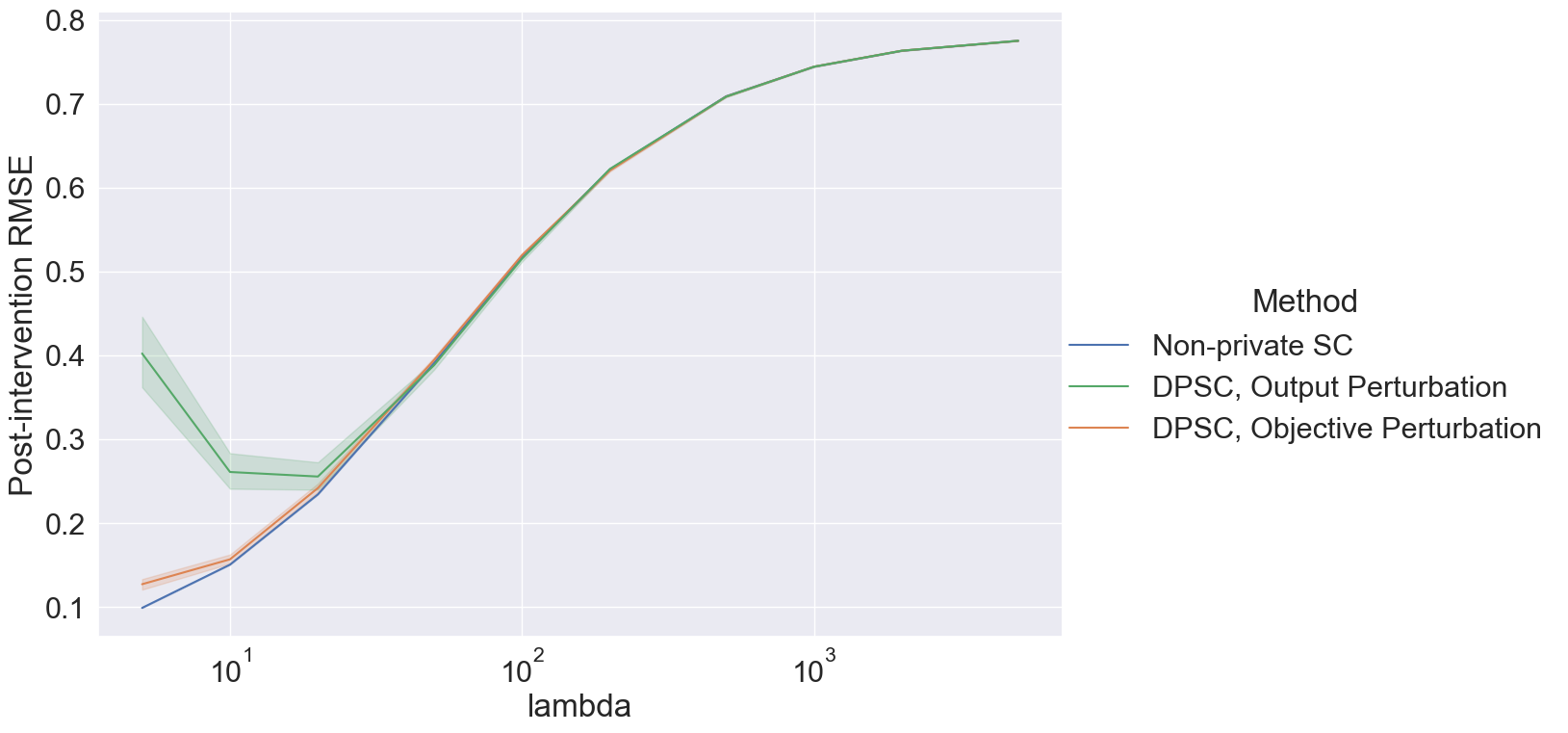}
\caption{
Behavior of post-intervention RMSE over $\lambda$, tested on a synthetic datast with $T_0=10, n=10$ for the synthetic control methods of non-private SC (blue), $DPSC_{out}$ (green), and $DPSC_{obj}$ (orange).}
\label{fig.exp.lambda.1}
\end{figure}

The U-shape of the curve for the private methods has a natural theoretical explanation: smaller $\lambda$ increases sensitivity and thus privacy noise and RMSE, while larger $\lambda$ increases the weight of the regularization term in the loss function, which will cause all three regularized methods to converge to the same value. Later, in Section \ref{s.exp.nandT0}, we observe that that the RMSE of DPSC is minimized around $\lambda=T_0$ for all four datasets of varying sizes (see Figure \ref{fig.exp.data.1}), which can aid the analyst in choosing an optimal $\lambda$.
This is consistent with our theoretical recommendations that $\lambda$ should be $O(T_0)$.

\begin{figure}[b]
\centering
\includegraphics[width=1\textwidth]{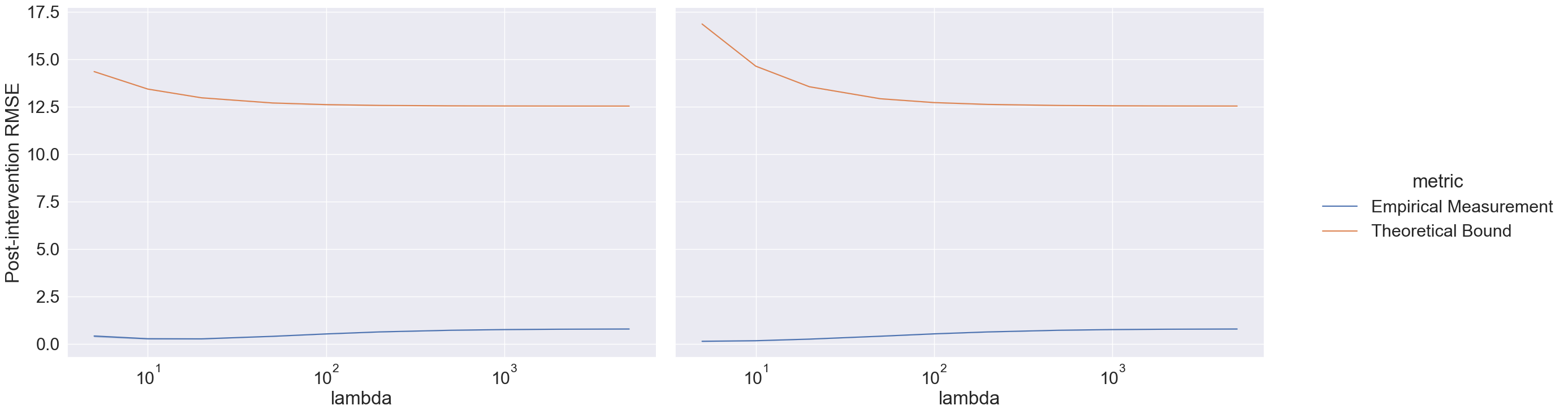}
\caption{Comparison of post-intervention RMSE in theory versus in practice, using $DPSC_{out}$ (left) and $DPSC_{obj}$ (right) on a dataset of size $n=10, T_0=10$.}
\label{fig.exp.lambda.2}
\end{figure}

Figure \ref{fig.exp.lambda.2} compares the empirical post-intervention RMSE of $DPSC_{out}$ and $DPSC_{obj}$ with the theoretical RMSE bounds of Theorem \ref{thm.acc.out} and \ref{thm.acc.obj} instantiated with the parameters values used in our experiments. We observe that $DPSC_{obj}$ performs better than $DPSC_{out}$ at smaller $\lambda$ values, although performance of both algorithms converges when $\lambda$ is large. We also observe that the empirical performance of both algorithms is dramatically better than the theoretical bounds predict, which suggests potential room for theoretical improvements.

\subsection{Effect of privacy parameters $\epsilon$ and $\delta$}
\label{s.exp.eps}

Next, we address the effect of $\eps$ in the performance of both $DPSC_{out}$ and $DPSC_{obj}$. In these experiments, we use $\lambda=T_0$ based on the findings in Section \ref{s.exp.lambda} and consider overall privacy budget $\eps=\eps_1 + \eps_2$ with $\eps_1 =\eps_2=\eps/2$. That is, the privacy budget is split evenly between the regression and projection steps in both algorithms. Results are presented for $\eps \in \{2, 4, 10, 20, 40, 100, 200\}$; stronger privacy guarantees (i.e., $\eps \leq 2$) were tested but excluded from the plots due to substantially higher RMSE values.

Figure \ref{fig.exp.eps.1} shows the post-intervention RMSE of $DPSC_{out}$ and $DPSC_{obj}$. As is to be expected, error diminishes with larger $\epsilon$. We also continue to observe $DPSC_{obj}$ outperforming $DPSC_{out}$ for most $\epsilon$ values, as in Section \ref{s.exp.lambda}. $DPSC_{out}$ performs slightly better thatn $DPSC_{obj}$ at $\eps=2$ in this dataset ($T_0=10$ and $n=10$); however, it is not the case for all datasets (see, e.g., Figure \ref{fig.exp.data.2} in Section \ref{s.exp.nandT0}).

\begin{figure}[h]
\centering
\includegraphics[width=0.75\textwidth]{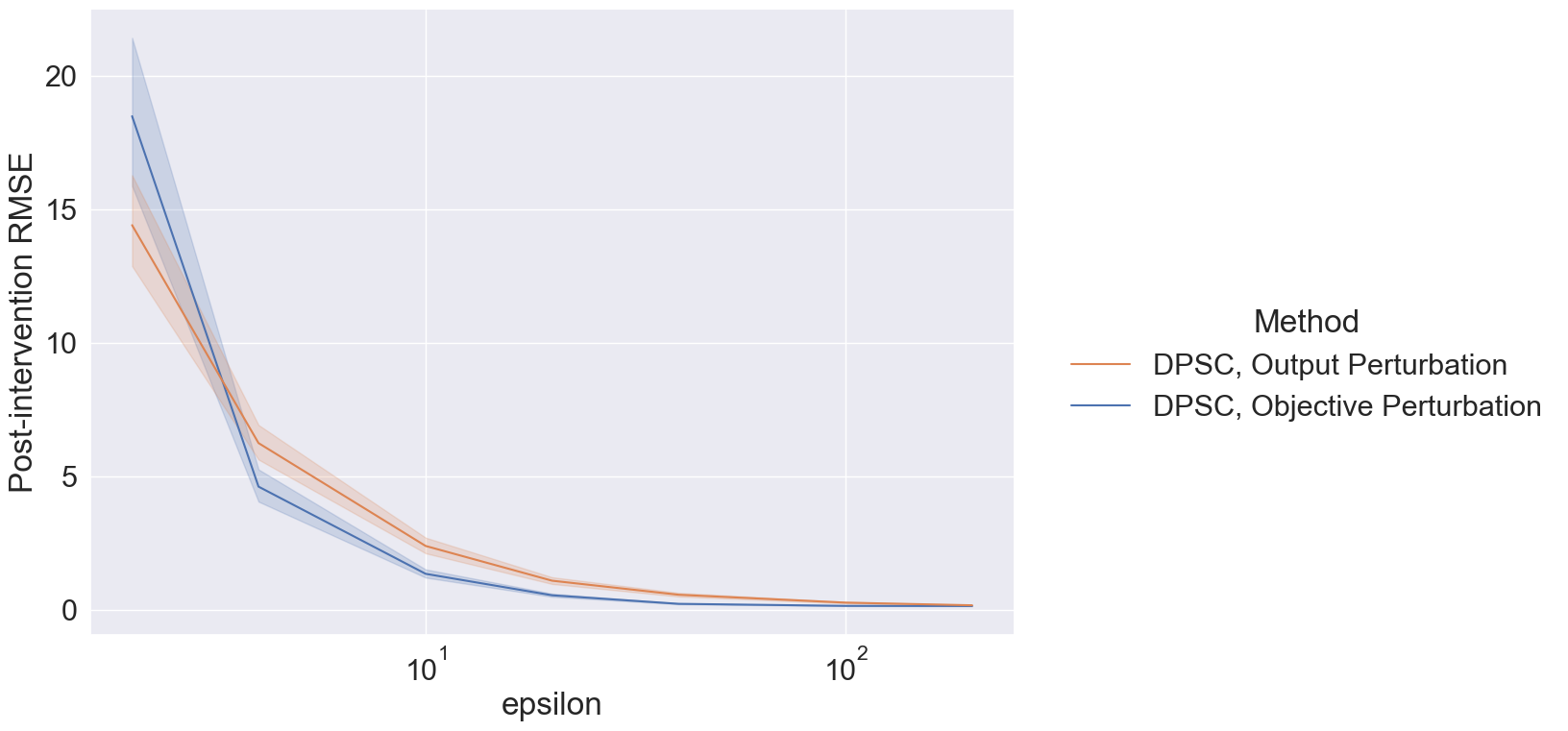}
\caption{Post-intervention RMSE of $DPSC_{out}$ (blue), and $DPSC_{obj}$ (red) for varying $\eps$, tested on a dataset with $T_0=10$ and $n=10$.}
\label{fig.exp.eps.1}
\end{figure}

For epsilon-regimes that are closer to the values chosen in practice (i.e., $\eps\leq 4$), the empirical RMSE was too high for practical use. We suggest a few methods to remedy this in future work.  First, a rejection sampling step can be introduced between the learning and projection steps of each algorithm that compares the noisy $\bm{f}^{out}$ and the original $\bm{f}^{reg}$. This step must also be done differentially privately to maintain the overall privacy guarantee. Additionally, our experiments only considered pure-DP with $\delta=0$; next we relax to approximate-DP with $\delta>0$ and observe that this yields lower RMSE than $\delta=0$.

\begin{figure}[h]
\centering
\includegraphics[width=0.96\textwidth]{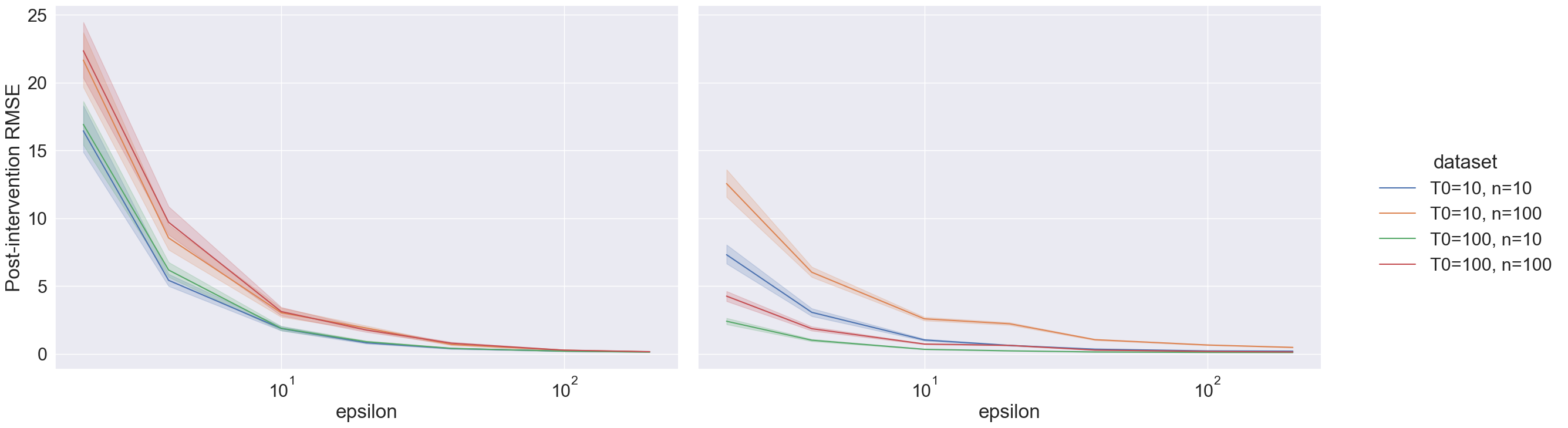}
\caption{Post-intervention RMSE of $DPSC_{obj}$  with Laplace noise (left, $\delta=0$) and Gaussian noise (right, $\delta=10^{-6}$) as a function of $\eps$ on four synthetic datasets of varying size.}
\label{fig.exp.delta}
\end{figure}

Keeping $\lambda = T_0$ fixed, Figure \ref{fig.exp.delta} shows the post-intervention RMSE of $DPSC_{obj}$ with Laplace noise (left) and with Gaussian noise (right, $\delta = 10^{-6}$). We observe improvements in RMSE with Gaussian noise for most epsilon values, especially for $\eps<10$. The impact was more evident in the datasets with a larger number of pre-intervention data points (i.e., $T_0=100$). This dependence on $T_0$ also changed the relative ordering of performance of each dataset; on the left with $\delta=0$, the two datasets with $n=10$ performed comparably had significantly lower RMSE than the two with $n=100$. On the right with $\delta>0$, the performance of all four datasets is stratified, now with the two datasets with $T_0=100$ performing the best. This evidence suggests meaningful empirical improvements can be obtained by setting a positive $\delta$ value, especially when $T_0$ is large.

\subsection{Impact of database size parameters $n$ and $T_0$}
\label{s.exp.nandT0}

Finally, we repeat the experiments of Sections \ref{s.exp.lambda} and \ref{s.exp.eps} with datasets of varying sizes. We consider four synthetic databases corresponding to combinations of $T_0\in\{10,100\}$ and $n\in\{10,100\}$, and evaluate the post-intervention RMSE under varying $\lambda$ and $\epsilon$.

Figure \ref{fig.exp.data.1} shows that the optimal choice of $\lambda$ roughly remains $\lambda=T_0$ for all sizes of database considered. The orange and blue curves in the figure corresponding to the two databases with $T_0=10$ are approximately minimized at $\lambda=10$, while the green and red curves with $T_0=100$ are approximately minimized at $\lambda=100$. While this is clearer in the figure for Objective Perturbation, it is also true for Output Perturbation, although the U-shape is less visible due to the larger scale of the $y$-axis.

\begin{figure}[h]
\centering
\includegraphics[width=0.95\textwidth]{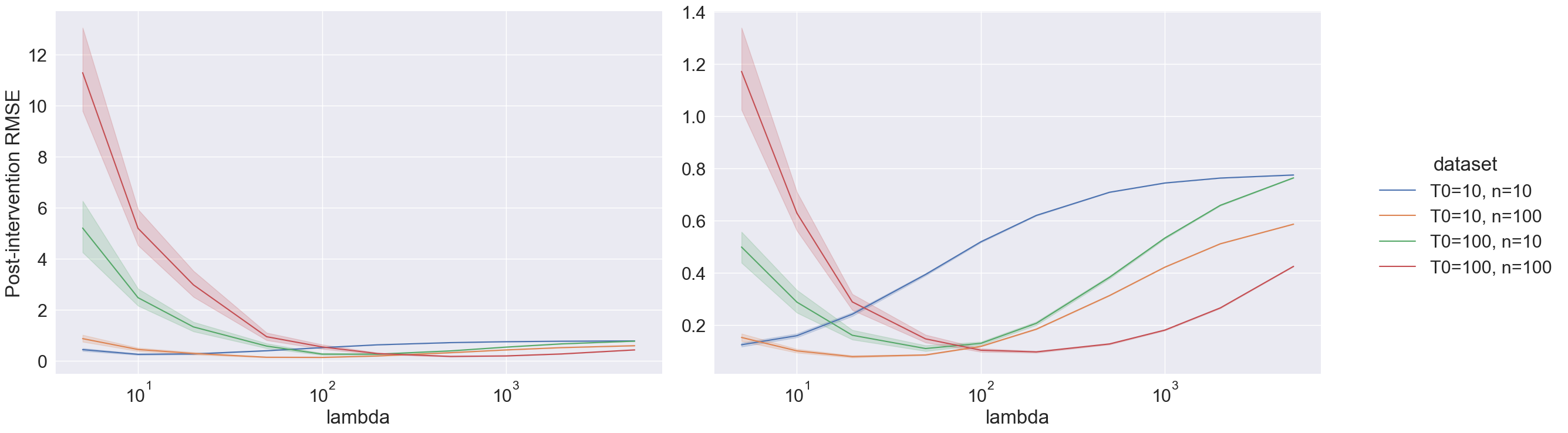}
\caption{Post-intervention RMSE of $DPSC_{out}$ (left) and $DPSC_{obj}$ (right) as a function of $\lambda$ on four synthetic datasets of varying size.}
\label{fig.exp.data.1}
\end{figure}

In Figure \ref{fig.exp.data.2}, we observe that the performance of $DPSC_{obj}$ (right) has less dependence on the size of the dataset than $DPSC_{out}$ (left), where we observe that smaller $n$ yields significantly better performance, especially when $\epsilon$ is not too large ($\epsilon < 100$). This difference in performance based on $n$ is consistent with our theoretical analysis in Section \ref{s.privacycost}, where we show that the post-intervention RMSE of $DPSC_{out}$ is $O(\frac{n}{\eps})$. For $DPSC_{obj}$, the accuracy bound has additional dependency on $T_0$ (Corollary \ref{cor.accuracy.obj}), which is less observable in these empirical results.

Overall, for $\eps \geq 4$, $DPSC_{obj}$ yields smaller RMSE than $DPSC_{out}$ for all datasets. However, at the smallest value of $\eps=2$, $DPSC_{out}$ slightly outperforms $DPSC_{obj}$ on the datasets with the smaller size of $n=10$. This suggests that $DPSC_{obj}$ should be the preferred algorithm, especially when $n$ is large, although other parameters may be relevant in this decision as well.

\begin{figure}[h]
\centering
\includegraphics[width=0.95\textwidth]{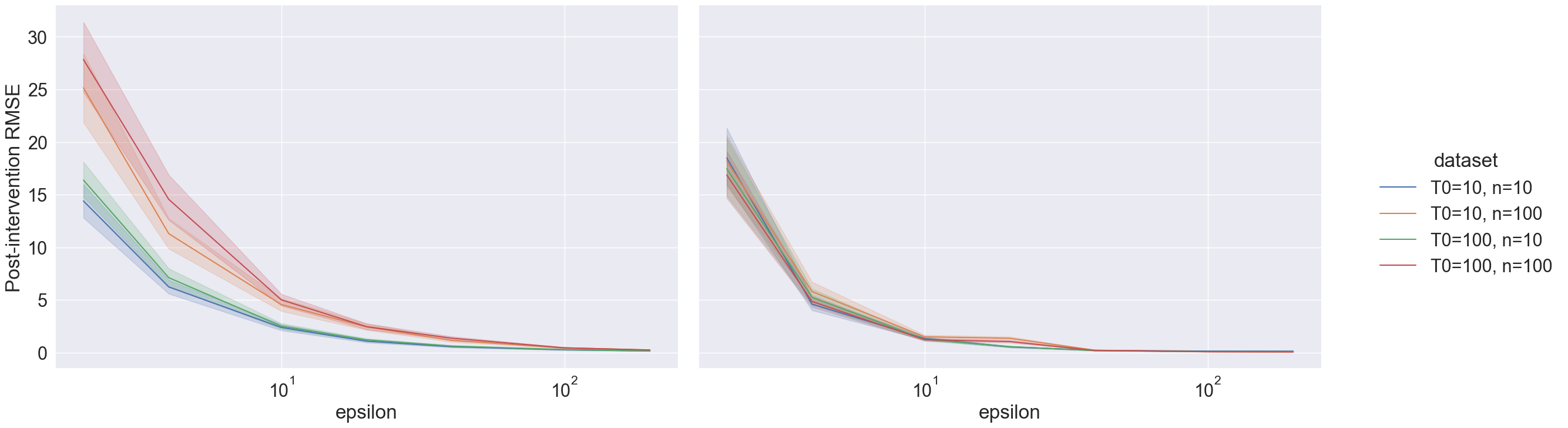}
\caption{Post-intervention RMSE of $DPSC_{out}$ (left) and $DPSC_{obj}$ (right) as a function of $\eps$ on four synthetic datasets of varying size.}
\label{fig.exp.data.2}
\end{figure}

\section{Conclusion}

As synthetic control is gaining popularity in medical applications where individual-level data are used as inputs, there is a growing need for private tools for synthetic control.
Synthetic control performs regression in a vertical way, making each time point one sample, rather than one user's data point. Existing approaches for private regression or private empirical risk minimization cannot be directly applied---the \emph{transposed} setting changes the definition of neighboring databases, altering the core sensitivity analysis needed for privacy.

This paper is the first to propose a differentially private versions of synthetic control algorithm. We provide algorithms based on output perturbation and objective perturbation, and provide formal privacy and accuracy guarantees for each. Our main technical contributions for both algorithms is a novel privacy analysis of the sensitivity of regression in the transposed setting, which also impacted our accuracy analysis and required novelty there as well. To enable practical use of our new private tools, we also provide tighter closed-form accuracy bounded for both algorithms under distributional assumptions, and guidance to practitioners for tuning the parameters of each algorithms.

We perform empirical evaluation of our algorithms to validate their performance guarantees in a variety of parameter regimes, as well as provide guidance to practitioners for hyperparameter tuning. We show that our algorithms perform even better than our theoretical bounds predict, which both suggests that our algorithms would perform well in practical deployments, and leaves an opportunity for further theoretical improvements in future work.

\section*{Acknowledgements}

\noindent S.R. and V.M. were supported in part by Novartis AG. 
R.C. and S.R. were supported in part by NSF grant CNS-2138834 (CAREER), a JPMorgan Chase Faculty Research Award, and an Apple Privacy-Preserving Machine Learning Award. Part of this work was completed while S.R. and R.C. were visiting the Simons Institute for the Theory of Computing.

\bibliographystyle{unsrt}  
\bibliography{references}  

\newpage

\appendix

\section{Notation}\label{app.notation}

\begin{table}[h]
\caption{Synthetic control notation}
\centering
\begin{tabular}{@{}cc@{}}
\toprule
Symbol    & Meaning                 \\ \toprule
$X = [X_{pre}, X_{post}]$  & The donor matrix, divided into pre- and post-intervention \\ \midrule
$\bm{y} = [\bm{y}_{pre}, \bm{y}_{post}]$ & The target vector, divided into pre- and post-intervention \\ \midrule
$n$ & The number of donors, i.e., the number of rows of $X$ \\ \midrule
$T_0$ & The number of pre-intervention time points \\ \midrule
$T$ & The number of full time points in donor matrix $X$, i.e., the number of columns of $X$\\ \midrule
$Z$ & The noise matrix i.i.d. from some distribution with zero-mean, $\sigma^2$-variance, and support $[-s,s]$\\ \midrule
$M$ & The true signal for donor matrix $X$ ($X=M+Z$)\\ \midrule
$\bm{z}$ & The noise vector i.i.d. from some distribution with zero-mean, $\sigma^2$-variance, and support $[-s,s]$\\ \midrule
$\bm{m}$ & The true signal for target vector $\bm{y}$ ($\bm{y}=\bm{m}+\bm{z}$)\\ \bottomrule
\end{tabular}
\end{table}

\begin{table}[h]
\caption{Empirical risk minimization notation}
\centering
\begin{tabular}{@{}ccc@{}}
\toprule
Objective function        & Formula    & Minimizer                  \\ \midrule
$\cL(\f)$ & $\frac{1}{T_0}||\bm{y} -X^\top \f||_2^2$ & - \\ \midrule
$J^{reg}(\f)$ & $\cL(\f)+\frac{\lambda}{2T_0}||\f||_2^2$ & $\f^{reg}$ \\ \midrule
$J^{\#}(\f)$ & $\cL(\f)+\frac{\lambda+\Delta}{2T_0}||\f||_2^2$ & $\f^{\#}$ \\ \midrule
$J^{obj}(\f)$  & $\cL(\f)+\frac{\lambda+\Delta}{2T_0}||\f||_2^2 +\frac{1}{T_0}b^\top \f$ & $\f^{obj}$               \\ \bottomrule
\end{tabular}
\end{table}

\section{Guidance for choosing $c$}\label{app.pickc}

For an analyst to use the objective perturbation method, one needs to decide what value to use for $c$. $c$ is the bound for the maximum absolute eigenvalue of the matrix $E = 2(X'_{pre}X'^{\top}_{pre} - X_{pre}X^{\top}_{pre})$, and a closed-form expression for $E$ is given in Equation \eqref{eq.E}. Note that $E$ is symmetric and all entries are zero except for one row and one column. WLOG, let the first column and row be the non-zero entries and call those entries $E_1, E_2, \cdots, E_n$. That is the matrix looks like:
\begin{equation*}
E=
    \begin{pmatrix}
        E_1 & E_2 & \cdots & E_n\\
        E_2 & 0   & \cdots & 0\\
        \vdots & \vdots & \ddots & \vdots\\
        E_n & 0 & \cdots & 0
\end{pmatrix}
\end{equation*}

To find eigenvalues $\lambda(E)$ of this matrix $E$, we want to solve $E \cdot \mathbf{v} = \lambda(E) \mathbf{v}$ for some $\mathbf{v} \in \mathbb{R}^n$ and $\lambda(E) \in \mathbb{R}$.
We obtain $n$ equations:
\begin{equation}\label{eq.findc}
    E_1 v_1 + E_2 v_2 + \cdots + E_n v_n = \lambda(E) v_1 
\end{equation}
and
\begin{equation*}
    E_i v_1 = \lambda(E) v_i \;, \forall i \in \{2, \cdots, n\}.
\end{equation*}

By plugging in $v_i = \frac{v_1}{\lambda} E_i$ for all $i\neq 1$ to Equation \eqref{eq.findc}, we get
\[
E_1v_1 + E_2^2\frac{v_1}{\lambda(E)} + \cdots + E_n^2\frac{v_1}{\lambda(E)} = \lambda(E) v_1.
\]
By removing $v_1$ and rearranging for $\lambda(E)$, this becomes a quadratic of $\lambda(E)$.
\[
\lambda(E)^2 - E_1\lambda(E) - \sum_{i=2}^n E_i^2 =0.
\]
Then, the two possible values for $\lambda(E)$ are:
\[
\lambda(E) = \frac{E_1 \pm \sqrt{E_1^2 + 4 \sum_{i=2}^n E_i^2}}{2}.
\]
Notice that the expression inside the square root is bigger than $E_1$, and we assume that $E_1 \geq 0$ without loss of generality. Then, the maximum absolute eigenvalue is 
\[
\lambda_{|max|}(E) = \frac{E_1 + \sqrt{E_1^2 + 4 \sum_{i=2}^n E_i^2}}{2}.
\]

Finally, we bound this by using the expression for $E$ in Equation \eqref{eq.E}, where we see $E_1 = 2(||\bm{x'_1}||_2^2 - ||\bm{x_1}||_2^2) \leq 2T_0$ (assuming $E_1\geq 0$) and $E_i = 2(\bm{x'_1} - \bm{x_1})^\top \bm{x_i} \leq 4T_0 $ for all $i \in \{2, \cdots, n\}$. Then,
$
\lambda_{|max|} \leq (1+ \sqrt{16n-15})T_0.
$ 
Hence, choosing $(c = 1+ \sqrt{16n-15})T_0$ is an option when the analyst does not have any knowledge other than $T_0$ and $n$ about the dataset. This value could be further minimized when there is more information available about the support of the dataset.

\section{Omitted Proofs for $DPSC_{out}$}\label{app.proofs}

\subsection{Proof of Lemma \ref{lem.goff}}

\goff*

\begin{proof}
We first re-arrange $g(\f)$ in a way that makes it easier to compute the gradient. Let $i$ be the index of the record that differs between $\mathcal{D}$ and $\mathcal{D'}$.
\begin{align}
    g(\f) &=
    \cL(\f,\mathcal{D'})-\cL(\f,\mathcal{D}) \notag\\ 
    &= \frac{1}{T_0}\sum_{t=1}^{T_0} \left[ \left(\sum_{k=1}^n {x'}_{k,t} f_k \right)-y_t\right]^2
    -\frac{1}{T_0}\sum_{t=1}^{T_0} \left[ \left(\sum_{k=1}^n x_{k,t} f_k \right)-y_t\right]^2\notag\\
    &= \frac{1}{T_0}\sum_{t=1}^{T_0} \left[ \left(\sum_{j\neq i} {x'}_{j,t} f_j \right)-y_t +{x'}_{i,t}f_i \right]^2
    -\frac{1}{T_0}\sum_{t=1}^{T_0} \left\{ \left(\sum_{j\neq i} x_{j,t} f_j \right)-y_t +x_{i,t}f_i \right\}^2 \notag\\
   &=\frac{1}{T_0}\sum_{t=1}^{T_0} \left[ 2\left(\sum_{j\neq i} x_{j,t} f_j -y_t\right) ({x'}_{i,t}-x_{i,t})f_i + ({x'}_{i,t}^2 - x_{i,t}^2)f_i^2 \right]\notag\\
    &=\frac{1}{T_0}\sum_{t=1}^{T_0} \left[
    2\left(\sum_{j\neq i} x_{j,t} f_j -y_t\right) + ({x'}_{i,t}+x_{i,t}) f_i
    \right]\left({x'}_{i,t}-x_{i,t} \right) f_i\notag\\
    &=\frac{1}{T_0}\sum_{t=1}^{T_0} \left[
    \left( \bm{x_t}^{\top} \f -y_t \right) +
    \left( \bm{x_t}^{\top} \f  -x_{i,t}f_i + {x'}_{i,t}f_i -y_t \right)
    \right]({x'}_{i,t}-x_{i,t})f_i
\label{eq-grad-1}
\end{align}

The second equality comes from the definition of the Ridge regression loss function $\cL$; in the third step we pull out the record $i$ that differs between $\D$ and $\D'$; the fourth step combines the sums and cancels terms, including the observation that $\sum_{j\neq i} x_{j,t} f_j=\sum_{j\neq i} x'_{j,t} f_j$. The final two steps also involve rearranging terms.

For notational ease, we define two additional terms,
\[
D_t := \bm{x_t}^{\top} \f -y_t \quad \text{ and } \quad
E_t := (x'_{i,t}-x_{i,t})f_i.
\]
Then, Equation \eqref{eq-grad-1} becomes
\[
g(\f)= \frac{1}{T_0}\sum_{t=1}^{T_0} (2D_t+E_t)E_t.
\]

We will take the partial derivatives of $D_t$ and $E_t$ with respect to both $f_i$ (the index of the data entry that differs between $\D$ and $\D'$) and $f_j$ for $j\neq i$, and then combine these to arrive at the gradient of $g(\f)$:
\[
\frac{\partial D_t}{\partial f_i} = x_{i,t}; \quad \frac{\partial D_t}{\partial f_j} = x_{j,t}; \quad \frac{\partial E_t}{\partial f_i} = x'_{i,t}-x_{i,t}; \quad
\frac{\partial E_t}{\partial f_j} = 0.
\]

Now, we compute the derivative of $g(\f)$ with respect to $f_i$.
\begin{align}
    \frac{\partial g(\f)}{\partial f_i} \notag
    &= \frac{1}{T_0}\sum_{t=1}^{T_0} \left\{ \left(2\frac{\partial D_t}{\partial f_i} + \frac{\partial E_t}{\partial f_i} \right)E_t
    + (2D_t+E_t) \frac{\partial E_t}{\partial f_i} \right\}\notag \\
    &= \frac{1}{T_0}\sum_{t=1}^{T_0} \left\{
    (x'_{i,t}+x_{i,t})E_t +(2D_t+E_t)(x'_{i,t}-x_{i,t})
    \right\} \notag\\
    &= \frac{1}{T_0}\sum_{t=1}^{T_0} \left\{ 2x'_{i,t}  (x'_{i,t}-x_{i,t})f_i + 2(\bm{x_t}^{\top} \f -y_t)(x'_{i,t}-x_{i,t}) \right\} \notag\\
    &= \frac{1}{T_0}\sum_{t=1}^{T_0} 2\left(x'_{i,t}f_i + \bm{x_t}^{\top} \f -y_t  \right) (x'_{i,t}-x_{i,t})
\label{eq-dgdfi}
\end{align}

Next, we compute the derivative of $g(\f)$ with respect to $f_j$ where $j$ is the index of unchanged donors ($j\neq i$). There are fewer term in this derivative because $\frac{\partial E_t}{\partial f_j}=0$.
\begin{align}
    \frac{\partial g(\f)}{\partial f_j}
    &= \frac{1}{T_0}\sum_{t=1}^{T_0} \left\{ \left(2x_{j,t}\right)E_t \right\} \notag\\
    &=\frac{1}{T_0}\sum_{t=1}^{T_0} 2x_{j,t} (x'_{i,t}-x_{i,t}) f_i 
\label{eq-dgdfj}
\end{align}

Finally, we can use (\ref{eq-dgdfi}) and (\ref{eq-dgdfj}) to derive an upper bound for $|| \nabla g(\f) ||_2$.

\begin{align}
\label{eq-gradient-g}
    || \nabla g(\f) ||_2^2 &= \left(\frac{\partial g(\f)}{\partial f_i} \right)^2 + \sum_{j\neq i} \left(\frac{\partial g(\f)}{\partial f_j}\right)^2 \notag \\
    &= \left( \frac{1}{T_0}\sum_{t=1}^{T_0} 2\left({x'}_{i,t}f_i + \bm{x_t}^{\top} \f -y_t  \right) ({x'}_{i,t}-x_{i,t}) \right)^2
    + \sum_{j \neq i} \left( \frac{1}{T_0}\sum_{t=1}^{T_0} 2x_{j,t} ({x'}_{i,t}-x_{i,t}) f_i \right)^2 \notag\\
    &\leq \frac{1}{T_0} \sum_{t=1}^{T_0} \left[ 2({x'}_{i,t}f_i + \bm{x_t}^{\top} \f -y_t  ) ({x'}_{i,t}-x_{i,t}) \right]^2
    + \sum_{j \neq i} \left[ \frac{1}{T_0}\sum_{t=1}^{T_0} \left[ 2x_{j,t} ({x'}_{i,t}-x_{i,t}) f_i \right]^2 \right]\notag\\
    & = \frac{1}{T_0} \sum_{t=1}^{T_0} \left[ 4({x'}_{i,t}f_i + \bm{x_t}^{\top} \f -y_t  )^2 ({x'}_{i,t}-x_{i,t})^2 + \sum_{j \neq i} 4 x_{j,t}^2 ({x'}_{i,t}-x_{i,t})^2 f_i^2 \right]\notag\\
    &= \frac{4}{T_0} \sum_{t=1}^{T_0}  ({x'}_{i,t}-x_{i,t})^2  \left[ ({x'}_{i,t}f_i + \bm{x_t}^{\top} \f -y_t  )^2 + \sum_{j \neq i} x_{j,t}^2f_i^2 \right]\notag\\
    &= \frac{4}{T_0} \sum_{t=1}^{T_0}  ({x'}_{i,t}-x_{i,t})^2  \left[ (\bm{x_t}^{\top} \f -y_t  )^2 + 2{x'}_{i,t}f_i(\bm{x_t}^{\top} \f -y_t) + \sum_{k=1}^n {{x'}}_{k,t}^2 f_i^2 \right]
\end{align}

The second equality comes from plugging in the partial derivatives computed in \eqref{eq-dgdfi} and \eqref{eq-dgdfj}, the following inequality comes from applying Jensen's inequality, and the final three steps come from rearranging, expanding, and simplifying terms.

We can proceed by bounding the individual terms in \eqref{eq-gradient-g} using the our modeling assumptions of Equation \eqref{eq.assump}, which give us that:
\[({x'}_{i,t}-x_{i,t})^2 \leq 4, \quad \text{ and } \quad (\bm{x_t}^{\top} \f -y_t  )^2 \leq 4,
\quad \text{ and } \quad
2{x'}_{i,t}f_i(\bm{x_t}^{\top} \f -y_t) \leq 4,
\quad \text{ and } \quad
\sum_{k=1}^n {x'}_{k,t}^2 f_i^2 \leq n.
\]

Then $|| \nabla g(\f) ||_2^2 \leq 128 + 16n$ and $|| \nabla g(\f) ||_2 \leq 4\sqrt{8 + n}$.

\end{proof}

\subsection{Proof of Lemma \ref{lem.threebounds}}\label{app.proofthreebounds}

\threebounds*

\begin{proof}
We prove these three bounds separately. Most steps follow from the sub-multiplicative norm property of Equation \eqref{eq.normbounds} and the bounds on the noise terms of Equation \eqref{eq.noisebounds}.

First,
\begin{align*}
    \mathbb{E}[|| M_{post}^{\top}(\f^{reg} -\f)||_2]
    &\leq \mathbb{E}[|| M_{post}||_{2,2} \cdot ||\f^{reg} -\f||_2]\\
    &\leq ||M_{post}||_{2,2} \mathbb{E}[\f^{reg} -\f||_2].
\end{align*}

Next,
\begin{align*}
    \mathbb{E}[||(Z_{post}^{\top} + W^{\top}) \f^{reg} ||_2]
    &\leq \mathbb{E}[||Z_{post} + W||_2 \cdot ||\f^{reg} ||_2]\\
    &\leq \mathbb{E}[||Z_{post} + W||_2 \cdot \sqrt{n}||\f^{reg} ||_{\infty}]\\
    &\leq \mathbb{E}[||Z_{post} + W||_2 \cdot \sqrt{n}\psi]\\
    &\leq \sqrt{n}\psi \cdot \mathbb{E}[||Z_{post}||_2 + ||W||_2]\\
    &\leq \sqrt{n}\psi \cdot \mathbb{E}[||Z_{post}||_F + ||W||_F]\\
    &\leq \sqrt{n} \psi \left(\sqrt{n(T-T_0) \sigma^2} + b \right)\\
    &\leq \sqrt{n} \psi \left(\sqrt{n(T-T_0) \sigma^2} + \frac{2\sqrt{T-T_0}}{\eps_2} \right),
\end{align*}
where the second step comes from the relationship between the $\ell_2$ norm and the $\ell_{\infty}$ norm, and the third step comes from our definition that $||\f^{reg} ||_{\infty} \leq \psi$ for some $\psi>0$.

Finally, 
\begin{align*}
    \mathbb{E}[||(M_{post}^{\top}+Z_{post}^{\top} + W^{\top})\bm{v}||_2 ]
    &\leq \mathbb{E}[||M_{post}+Z_{post} + W||_{2,2} ||\bm{v}||_2 ]\\
    &= \mathbb{E}[||M_{post}+Z_{post} + W||_{2,2}] \cdot \mathbb{E}[||\bm{v}||_2 ]\\
    &\leq \mathbb{E}[||M_{post}||_{2,2}+||Z_{post}||_{2,2} + ||W||_{2,2}] \cdot  \frac{4\sqrt{8+n}}{\lambda \eps_1}\\
    &\leq \left( ||M_{post}||_{2,2} +\mathbb{E}[||Z_{post}||_F + ||W||_F] \right) \cdot  \frac{4\sqrt{8+n}}{\lambda \eps_1}\\
    &\leq \left(||M_{post}||_{2,2} + \sqrt{n(T-T_0) \sigma^2}+\frac{2\sqrt{T-T_0}}{\eps_2}\right) \frac{4\sqrt{8+n}}{\lambda \eps_1},
\end{align*}
where the second step holds because $Z_{post}$, $W$ and $\bm{v}$ are all independent of each other.
\end{proof}

\subsection{Proof of Lemma \ref{lem.estbounds}}\label{app.lemestbounds}

\estbounds*

\begin{proof}
First we can expand $\mathbb{E}[||\f^{reg}-\f||_2]$:
\begin{align}\label{eq.f.l2}
    \mathbb{E}[||\f^{reg}-\f||_2] &=  \mathbb{E}[||\f^{reg} -\mathbb{E}[\f^{reg}] +\mathbb{E}[\f^{reg}] - \f||_2] \notag \\
    &\leq \mathbb{E}[||\f^{reg} -\mathbb{E}[\f^{reg}] ||_2] + \mathbb{E}[||\mathbb{E}[\f^{reg}] - \f||_2] \notag\\
    &= \mathbb{E}[||\f^{reg} -\mathbb{E}[\f^{reg}] ||_2] + \mathbb{E}[||Bias(\f^{reg})||_2],
\end{align}
where
\[
Bias(\f^{reg})= \mathbb{E}[\f^{reg}] - \f = -\lambda (X_{pre} X_{pre}^{\top} +\lambda I)^{-1} \f.
\]
Hence, we only need to bound the two terms: $||Bias(\f^{reg})||_2$ and
$\mathbb{E}[|| \f^{reg} - \mathbb{E}[\f^{reg}] ||_2]$, which we do next. First,

\begin{align*}
    ||Bias(\f^{reg})||_2 &= || -\lambda (X_{pre} X_{pre}^{\top} +\lambda I)^{-1} \f||_2\\
    & \leq \lambda ||\f||_2 ||(X_{pre} X_{pre}^{\top} +\lambda I)^{-1}||_{2,2}\\
    & \leq \lambda ||(X_{pre} X_{pre}^{\top} +\lambda I)^{-1}||_{2,2},
\end{align*}
where the last inequality uses the fact that the $\ell_1$ norm of $\f$ is $1$, which also upper bound the $\ell_2$ norm. Next,
\begin{align*}
    \mathbb{E}[|| \f^{reg} - \mathbb{E}[\f^{reg}] ||_2] &= \mathbb{E}[|| \f^{reg} - (\f + Bias(\f^{reg})) ||_2] \\
    &= \mathbb{E}[||(X_{pre} X_{pre}^{\top} +\lambda I)^{-1} X_{pre}\bm{y_{pre}} - \f + \lambda (X_{pre} X_{pre}^{\top} +\lambda I)^{-1} \f ||_2]\\
    &= \mathbb{E}[||(X_{pre} X_{pre}^{\top} +\lambda I)^{-1} X_{pre}(M_{pre}^{\top} \f + \bm{z_{pre}}) - \f + \lambda (X_{pre} X_{pre}^{\top} +\lambda I)^{-1} \f ||_2]\\
    &= \mathbb{E}[||(X_{pre} X_{pre}^{\top} +\lambda I)^{-1} X_{pre}(X_{pre}^{\top} \f - Z_{pre}^{\top}\f + \bm{z_{pre}} ) - \f + \lambda (X_{pre} X_{pre}^{\top} +\lambda I)^{-1} \f ||_2]\\
    &= \mathbb{E}[|| (X_{pre} X_{pre}^{\top} +\lambda I)^{-1} (X_{pre} X_{pre}^{\top} +\lambda I) \f - \f + (X_{pre} X_{pre}^{\top} +\lambda I)^{-1} (X_{pre}\bm{z_{pre}} -X_{pre}Z_{pre}^{\top}\f) ||_2]\\
    &=\mathbb{E}[|| (X_{pre} X_{pre}^{\top} +\lambda I)^{-1} (X_{pre}\bm{z_{pre}} -X_{pre}Z_{pre}^{\top}\f) ||_2]\\
    &\leq \mathbb{E}[|| (X_{pre} X_{pre}^{\top} +\lambda I)^{-1}||_{2,2} \cdot ||X_{pre}\bm{z_{pre}} -X_{pre}Z_{pre}^{\top}\f ||_2],
\end{align*}
where the first four steps come respectively from plugging in expressions of $\mathbb{E}[\f^{reg}]$, ($\f^{reg}$ and $Bias(\f^{reg})$), $\bm{y_{pre}}$, and $M_{pre}$. The fifth and sixth steps come from rearranging and canceling terms, and the final inequality comes from the submultiplicative norm property of Equation \eqref{eq.normbounds}.

Plugging everything back to \eqref{eq.f.l2} yields,
\begin{align}\label{eq.estbound.interim.1}
    \mathbb{E}[||\f^{reg}-\f||_2] &\leq
    \mathbb{E}[||\f^{reg} -\mathbb{E}[\f^{reg}] ||_2] + \mathbb{E}[||Bias(\f^{reg})||_2] \notag\\
    &\leq \mathbb{E}[|| (X_{pre} X_{pre}^{\top} +\lambda I)^{-1}||_{2,2} \cdot ||X_{pre}\bm{z_{pre}} -X_{pre}Z_{pre}^{\top}\f ||_2+\lambda ||(X_{pre} X_{pre}^{\top} +\lambda I)^{-1}||_{2,2}] \notag\\
    &= \mathbb{E}[|| (X_{pre} X_{pre}^{\top} +\lambda I)^{-1}||_{2,2} \cdot ( ||X_{pre}\bm{z_{pre}} -X_{pre}Z_{pre}^{\top}\f ||_2+\lambda )]
\end{align}

Next, we use our assumptions on the data distribution to prove the following lemma about $|| (X_{pre} X_{pre}^{\top} +\lambda I)^{-1}||_{2,2}$.

\eigenlemma*

\begin{proof}[Proof of Lemma \ref{lem.eigen}]
A key component of the proof of Lemma \ref{lem.eigen} is the following lemma about concentration of random matrices. 

\corollarylemma*

To instantiate Lemma \ref{lem.cor552}, we view the data $X_{pre}$ as $T_0$ samples corresponding to the columns $ \bm{x_t} \in \mathbb{R}^n, \; \forall t\in \{1, 2, \cdots, T_0\}$.
We use our assumptions that $X$ takes values in a $k$-dimensional subspace $E$, and $\Sigma = P_E$ where $P_E$ is the orthogonal projection from $\mathbb{R}^n$ onto $E$. 
Then, the \emph{effective rank} of $\Sigma$ is $r(\Sigma) = \frac{trace(\Sigma)}{||\Sigma||_2} = k$ by definition, because $||\Sigma||_{2,2} =||P_E||_{2,2}= \bm{\sigma}_{max}(P_E) = \bm{\lambda}_{max}(P_E) =1$, since eigenvalues of an orthogonal projection matrix are either $0$ or $1$ as shown in Lemma 19 of \cite{rsc}.
Then, $\mathbb{E}[||X||_{2,2}^2] = trace(\Sigma) = k ||\Sigma||_{2,2} = k ||P_E||_{2,2} = k$.
Using Markov's inequality, most of the distribution should be within a ball of radius $\sqrt{m}$ where $m = O(k)$.
Finally, let us assume that all the probability mass is within that ball, i.e., $||X||_{2,2}=O(\sqrt{k})$ almost surely.  Then, Lemma \ref{lem.cor552} holds with $T_0 \geq C(t/\eps)^2 k \log n$ samples. This is also noted in Remark 5.53 of \cite{vershynin2010introduction}.

To translate this to our setting, we see that with probability at least $1-n^{-t^2}$, if $T_0 \geq  C(t/\xi)^2  k \log n$, then
\begin{equation}\label{cor.552.applied}
||\frac{1}{T_0}X_{pre}X_{pre}^{\top} - \Sigma||_{2,2} \leq \xi.
\end{equation}
Since $\Sigma = P_E$ is an orthogonal projection matrix, $||P_E||_{2,2}=1$. We apply triangle inequality to obtain,
\[
||\frac{1}{T_0}X_{pre}X_{pre}^{\top} - P_E||_{2,2} 
\geq \left| ||\frac{1}{T_0}X_{pre}X_{pre}^{\top}||_{2,2}  - ||P_E||_{2,2} \right| = \left| ||\frac{1}{T_0}X_{pre}X_{pre}^{\top}||_{2,2}  - 1 \right| \geq ||\frac{1}{T_0}X_{pre}X_{pre}^{\top} - I ||_{2,2} .
\]
Combining this with Equation \eqref{cor.552.applied}, we can bound
\begin{equation}\label{eq.upperbound}
||\frac{1}{T_0}X_{pre}X_{pre}^{\top} - I||_{2,2} \leq \xi, \quad \text{or equivalently,} \quad || X_{pre}X_{pre}^{\top} - T_0 I ||_{2,2} \leq \xi T_0.
\end{equation}
We will use this latter expression to obtain a lower bound on the minimum singular value of $X_{pre}X_{pre}^{\top}$, and then use it to bound $||(X_{pre}X_{pre}^{\top} +\lambda I)^{-1} ||_{2,2}$ from above.

Note that since $||A||_{2,2}$ is the maximum singular value of matrix $A$, the upper bound of $\xi T_0$ of Equation \eqref{eq.upperbound} should hold for all singular values of $A$. For symmetric matrices such as $X_{pre}X_{pre}^{\top} +T_0 I$, the singular values are also the absolute values of its eigenvalues.
This means that all eigenvalues $\bm{\lambda_{\star}}$ of $X_{pre}X_{pre}^{\top} -  T_0I$ must satisfy $|\bm{\lambda_{\star}}(X_{pre}X_{pre}^{\top} -  T_0 I )| \leq \xi T_0$. 
Therefore, this bound must also hold for the smallest eigenvalue $\bm{\lambda_{min}}(\cdot)$:
\begin{align*}
    |\bm{\lambda_{min}}(X_{pre}X_{pre}^{\top} - T_0 I)| & \leq \xi T_0\\
  \Longleftrightarrow \qquad \quad   |\bm{\lambda_{min}}(X_{pre}X_{pre}^{\top}) -  T_0|  &\leq \xi T_0 \\
   \Longleftrightarrow \; \;    (1- \xi) T_0  \leq \bm{\lambda_{min}}(X_{pre}X_{pre}^{\top}) &\leq  (1+ \xi) T_0
\end{align*}
By plugging in the lower bound on the minimum singular value of $X_{pre}X_{pre}^{\top}$, we arrive at the desired bound to complete the proof of Lemma \ref{lem.eigen}.
\begin{align*}
    || (X_{pre}X_{pre}^{\top} + \lambda I)^{-1} ||_2
    &= \bm{\sigma_{max}}((X_{pre}X_{pre}^{\top} +\lambda I)^{-1}) \\
    &= \frac{1}{\bm{\sigma_{min}}(X_{pre}X_{pre}^{\top} +\lambda I)}\\
    &= \frac{1}{|\bm{\lambda_{min}}(X^{\top}X) +\lambda|}\\
    &\leq \frac{1}{ (1-\xi)T_0 +\lambda}.
\end{align*}

\end{proof}

Returning to Equation \eqref{eq.estbound.interim.1}, we can use this bound to obtain,
\begin{align}\label{eq.estbound.interim.2}
    \mathbb{E}[||\f^{reg}-\f||_2] 
    &\leq \mathbb{E}[|| (X_{pre} X_{pre}^{\top} +\lambda I)^{-1}||_{2,2} \cdot ( ||X_{pre}\bm{z} -X_{pre}Z^{\top}\f ||_2+\lambda )] \notag\\
    &\leq \frac{1}{ (1-\xi)T_0 +\lambda} \mathbb{E}[ ||X_{pre}\bm{z} -X_{pre}Z^{\top}\f ||_2+\lambda ]
\end{align}

The expectation term in Equation \eqref{eq.estbound.interim.2} becomes,
\begin{align*}
    \mathbb{E}[ ||X_{pre}\bm{z_{pre}} -X_{pre}Z_{pre}^{\top}\f ||_2 + \lambda]
    &= \mathbb{E}[ || (M_{pre}+Z_{pre})\bm{z_{pre}} - (M+Z_{pre})Z_{pre}^{\top}\f ||_2 +\lambda ]\\
    &\leq \mathbb{E}[ || M_{pre}(\bm{z_{pre}} - Z_{pre}^{\top}\f)||_2 + || Z_{pre}(\bm{z_{pre}}-Z_{pre}^{\top}\f) ||_2 +\lambda ]\\
    &\leq ||M_{pre}||_{2,2} \mathbb{E}[ ||\bm{z_{pre}} - Z_{pre}^{\top}\f||_2] + \mathbb{E}[||Z_{pre}||_{2,2} \cdot ||\bm{z_{pre}} - Z_{pre}^{\top}\f||_2] +\lambda \\
    &= ||M_{pre}||_F \mathbb{E}[ ||\bm{z_{pre}} - Z_{pre}^{\top}\f||_2] + \mathbb{E}[||Z_{pre}||_F \cdot ||\bm{z_{pre}} - Z_{pre}^{\top}\f||_2] +\lambda \\
    &\leq \sqrt{n T_0} \mathbb{E}[ ||\bm{z_{pre}} - Z_{pre}^{\top}\f||_2] + \sqrt{n T_0 s^2}\mathbb{E}[||\bm{z_{pre}} - Z_{pre}^{\top}\f||_2] +\lambda \\
    &= (\sqrt{n T_0}  + \sqrt{n T_0 s^2})\mathbb{E}[||\bm{z_{pre}} - Z_{pre}^{\top}\f||_2] +\lambda,
\end{align*}
where the first step is plugging in for $X_{pre}$, the second step is triangle inequality, the third and fourth steps are due to the submultiplicative norm property, the fifth step comes from the definition of the Frobenius norm, the fact that $M_{pre}$ and $Z_{pre}$ are both of dimension $n \times T_0$, and bounds on data entries. The final step collects terms.

Finally, we need only to obtain a bound on $\mathbb{E}[||\bm{z_{pre}} - Z_{pre}^{\top}\f||_2]$.
\begin{align*}
    \mathbb{E}[||\bm{z_{pre}} - Z_{pre}^{\top}\f||_2]
    &\leq \mathbb{E} \left[ \sqrt{\sum_{t=1}^{T_0} (z_t-Z_t^{\top}\f)^2 } \right]\\
    &\stackrel{(a)}{\leq} \sqrt{\sum_{t=1}^{T_0} \mathbb{E}[(z_t-Z_t^{\top}\f)^2 ]}\\
    &= \sqrt{\sum_{t=1}^{T_0} \mathbb{E}[z_t^2 - 2z_t Z_t^{\top}\f + (Z_t^{\top}\f)^2 ]}\\
    &\stackrel{(b)}{=} \sqrt{\sum_{t=1}^{T_0} (\sigma^2 +\mathbb{E}[ (Z_t^{\top}\f)^2 ] )}\\
    &= \sqrt{T_0\sigma^2 +\sum_{t=1}^{T_0} \mathbb{E}[ \sum_{i=1}^n(z_i f_i)^2 ] )}\\
    &\stackrel{(c)}{=} \sqrt{T_0\sigma^2 +\sum_{t=1}^{T_0} \sum_{i=1}^n \mathbb{E}[ z_i^2 f_i^2 ] )}\\
    &= \sqrt{T_0\sigma^2 +\sum_{t=1}^{T_0} \sum_{i=1}^n \sigma^2 f_i^2 }\\
    &= \sqrt{T_0\sigma^2 +\sum_{t=1}^{T_0} \sigma^2 ||\f||_2^2 }\\
    &\stackrel{(d)}{\leq} \sqrt{T_0\sigma^2 +\sum_{t=1}^{T_0} \sigma^2}\\
    &= \sqrt{2T_0 \sigma^2}
\end{align*}

Inequality $(a)$ is due to Jensen's inequality.
The step in $(b)$ is because $\mathbb{E}[z_t Z_t^{\top}\f]=\mathbb{E}[z_t]\mathbb{E}[Z_t^{\top}\f]=0$ by independence of noise terms.
The step in $(c)$ is by the same logic as in $(b)$, since all cross-terms $f_i f_j$ for  $i\neq j$ are zero in expectation.
Lastly, we bound the $\ell_2$ norm of $\f$ by $\ell_1$ norm instead in $(d)$ (i.e., $||\f||_2 \leq ||\f||_1 \leq 1$).

Hence, 
\begin{align*}
    \mathbb{E}[ ||X_{pre}\bm{z_{pre}} -X_{pre}Z_{pre}^{\top}\f ||_2 + \lambda]
    &\leq (\sqrt{n T_0}  + \sqrt{n T_0 s^2})\sqrt{2T_0 \sigma^2}+\lambda \\
    &= T_0\sqrt{2n\sigma^2} + T_0\sqrt{2n\sigma^2s^2} + \lambda
\end{align*}

Finally, combining this with Equation \eqref{eq.estbound.interim.2} gives the desired bound to complete the proof of Lemma \ref{lem.estbounds}.
\[    \mathbb{E}[||\f^{reg}-\f||_2]
    \leq \frac{ (\sqrt{2n\sigma^2} + \sqrt{2n\sigma^2s^2})T_0 + \lambda}{ (1-\xi)T_0 +\lambda}.
\]

\end{proof}

\section{Omitted Proofs for $DPSC_{obj}$}\label{app.proofsobj}

\subsection{Proof of Lemma \ref{lem.step2}}\label{app.lemstep2}

\steptwo*

\begin{proof}
Recall that $\bm{b}(\bm{\alpha};\D)$ is the noise value that must have been realized when database $\D$ was input and $\bm{\alpha} =\arg\min_{\f} J^{obj}(\f)$ was output. Since $J^{obj}(\f)$ is strongly convex for any $\Delta$  and is differentiable, the closed-form expression for $\bm{b}(\bm{\alpha};\D)$ is derived by computing the gradient of $J^{obj}(\f)$, which should be zero when evaluated at its minimizer $\f = \bm{\alpha}$:
\[\nabla J^{obj}(\f) \big|_{\f = \bm{\alpha}} = \nabla\cL(\bm{\alpha}) +
\nabla r(\bm{\alpha})+
\frac{\Delta}{T_0}\bm{\alpha} + \frac{\bm{b}(\bm{\alpha};\D)}{T_0} \stackrel{!}{=} 0.\]
Rearranging the equation yields
\[
\bm{b}(\bm{\alpha};\D) = -\left( T_0 \nabla\cL(\bm{\alpha};\D) + T_0 \nabla r(\bm{\alpha}) + \Delta \bm{\alpha} \right).
\]

For ease of notation, let $A = -\nabla\bm{b}(\alpha;\D)$ and $E = \nabla\bm{b}(\alpha;\D)-\nabla\bm{b}(\alpha;\D')$. Then, 
\[
\Phi(\bm{\alpha};\Delta) = \frac{|det(\nabla\bm{b}(\bm{\alpha};\D'))|}{|det(\nabla\bm{b}(\bm{\alpha};\D))|} = \frac{|det(-\nabla\bm{b}(\bm{\alpha};\D'))|}{|det(-\nabla\bm{b}(\bm{\alpha};\D))|} = \frac{|det(A+E)|}{|det(A)|}.
\]
By definition, $A = -\nabla\bm{b}(\alpha;\D) = T_0 ( \nabla^2 \cL(\alpha;\D) + \nabla^2 r(\alpha)) + \Delta I_n$. Using the Hessians $\nabla^2 \cL(\alpha;\D) = \frac{2}{T_0} X_{pre}X_{pre}^{\top}$ and $\nabla^2 r(\alpha) = \frac{\lambda}{T_0} I_n $, $A$ can be expressed as
\begin{equation*}
        A = 2X_{pre}X_{pre}^T + (\lambda + \Delta) I_n.
\end{equation*}
To express $E$ succinctly, let neighboring databases $\D = (X, y)$ and $\D' = (X',y)$ differ in the $j$-th row. Then, 
\begin{equation}\label{eq.E}
        E = 2(X'_{pre}X'^{\top}_{pre} - X_{pre}X^{\top}_{pre})
        =  \begin{cases}
              2(||\bm{x'_j}||_2^2 - ||\bm{x_j}||_2^2)  & (j,j) \\
              2(\bm{x'_j} - \bm{x_j})^\top \bm{x_i}  & (j,i) \text{ or } (i,j) \;, \; \forall i \in [n], \; i\neq j \\
              0 & \text{otherwise}
            \end{cases}
\end{equation}
where $\bm{x_i}$ (resp. $\bm{x'_i}$) denotes the $i$-th person's data, which is the $i$-th row of $X_{pre}$ (resp. $X'_{pre}$).

Note that all eigenvalues of $A$ are at least $\lambda + \Delta > 0$ (i.e., $\lambda_{min}(A)\geq \lambda + \Delta$) because $X_{pre}X^{\top}_{pre}$ is positive-semi-definite, and thus $A$ is full rank. Also, $rank(E)=2$. This allows us to apply the following lemma.

\begin{lemma}[Lemma 2 of \cite{CMS11}]\label{lem.determinant}
If $A$ is full rank and $E$ has rank at most $2$,
\[
\frac{det(A+E) - det(A)}{det(A)} = \lambda_1(A^{-1}E) + \lambda_2(A^{-1}E) + \lambda_1(A^{-1}E)\lambda_2(A^{-1}E),
\]
where $\lambda_i(Z)$ is $i$-th eigenvalue of matrix $Z$.
\end{lemma}
Let $\lambda_{|max|}(Z) = \max_i |\lambda_i(Z)|$, the maximum absolute of eigenvalue of matrix $Z$.
Instantiating Lemma \ref{lem.determinant} yields: 
\begin{equation*}
    \begin{split}
        \Phi(\bm{\alpha};\Delta) &= \frac{|det(A+E)|}{|det(A)|}\\
        &= \left|\frac{det(A+E) - det(A)}{det(A)} + 1\right|\\
        &= | 1+ \lambda_1(A^{-1}E) + \lambda_2(A^{-1}E) + \lambda_1(A^{-1}E)\lambda_2(A^{-1}E)|\\
        &\leq 1 + |\lambda_1(A^{-1}E)| + |\lambda_2(A^{-1}E)| + |\lambda_1(A^{-1}E)\lambda_2(A^{-1}E)|\\
        &\leq 1 + 2\lambda_{|max|}(A^{-1}E) + \lambda_{|max|}(A^{-1}E)^2,
    \end{split}
\end{equation*}
where the first inequality is simply triangle inequality, and the second inequality bounds all absolute eigenvalues by the maximum one $\lambda_{|max|}$.

Assume that $\lambda_{|max|}(E)\leq c$ for some constant $c$. Since $E$ is a real-valued matrix, such a finite $c$ exist. In Algorithm \ref{alg.obj}, $c$ is explicitly taken as an input parameter.
Then,
\[
\lambda_{|max|}(A^{-1}E) \leq \frac{\lambda_{|max|}(E)}{\lambda_{min}(A)} \leq \frac{c}{\lambda+\Delta}.
\]
Finally,
\[
\Phi(\bm{\alpha};\Delta) \leq 1 + 2\lambda_{|max|}(A^{-1}E) + \lambda_{|max|}(A^{-1}E)^2 \leq 1 + \frac{2c}{\lambda+\Delta} + \frac{c^2}{(\lambda+\Delta)^2} \leq \left( 1+ \frac{c}{\lambda+\Delta} \right)^2.
\]
\end{proof}

\subsection{Proof of Lemma \ref{lem.laplacenoise}}\label{app.lemlaplace}

\laplacenoise*

\begin{proof}
We can start by re-writing $\Gamma(\bm{\alpha})$ as follows, where the first line directly comes from the pdf $\Pr(\bm{b};\beta)$, the second line is due to reverse triangle inequality, and the third line is from the definition of $\bm{b}(\bm{\alpha};\mathcal{D})$ and canceling terms that occur in both $\bm{b}(\bm{\alpha};\mathcal{D})$ and $\bm{b}(\bm{\alpha};\mathcal{D}')$:
\begin{align}\label{eq.boundongamma}
\Gamma(\bm{\alpha}) = &\exp{\left( -\frac{1}{\beta} \bigl | \; ||\bm{b}(\bm{\alpha} ; \D)||_2 - ||\bm{b}(\bm{\alpha} ; \D')||_2 \; \bigr |  \right)}\notag \\
\leq &\exp{\left( -\frac{1}{\beta} ||\bm{b}(\bm{\alpha};\D) - \bm{b}(\bm{\alpha};\D')||_2  \right)} \notag \\
= &\exp{\left(\frac{1}{\beta} || T_0 \nabla\cL(\bm{\alpha};\D') - T_0 \nabla\cL(\bm{\alpha};\D)||_2  \right)}.
\end{align}

Next, we can continue to bound Equation \eqref{eq.boundongamma} in two different ways, corresponding to the two possible values of $\beta$. The two values come from two different upper bounds on the sensitivity, and the minimum value will give a tighter bound. 

The first upper bound uses Lemma \ref{lem.goff}, and its notation of $g(\f) = \cL(\f,\mathcal{D'})
- \cL(\f,\mathcal{D})$ for neighboring databases $\D,\D'$. Then we can bound:
\begin{align*}
    \eqref{eq.boundongamma} & \leq \exp{\left(\frac{1}{\beta} ||T_0\nabla g(\bm{\alpha})||_2   \right)}\\
&\leq \exp{\left(\frac{1}{\beta} 4T_0\sqrt{8+n}  \right)}.
\end{align*}
Hence, setting $\beta\geq \frac{4T_0\sqrt{8+n}}{\eps_0}$ makes $\Gamma(\alpha)\leq e^{\eps_0}$.

The second upper bound is based on $c$, and will yield a tighter bound when $c$ is small. Recall that matrix $E$ is defined in Equation \eqref{eq.E}, and that $c$ is the upper bound $\bm{\lambda}_{|max|}(E)\leq c$.
By plugging in $\nabla\cL(\alpha) = \frac{1}{T_0} \left( 2X_{pre}X_{pre}^\top \bm{\alpha} - 2X_{pre}\bm{y}_{pre} \right)$, we can alternatively bound:
\begin{align*}
    \eqref{eq.boundongamma} &= \exp{ \left(\frac{1}{\beta} ||2(X'_{pre}X'^{\top}_{pre} - X_{pre}X^{\top}_{pre})\bm{\alpha}  + 2(X_{pre}-X'_{pre})\bm{y}_{pre}||_2\right)}\\
    &\leq \exp{ \left(\frac{1}{\beta} ||2(X'_{pre}X'^{\top}_{pre} - X_{pre}X^{\top}_{pre})\bm{\alpha}||_2 + \frac{1}{\beta}|| 2(X'_{pre}-X_{pre})\bm{y}_{pre}||_2\right)}\\
    &\leq \exp{\left(\frac{1}{\beta}||E\bm{\alpha}||_2 + \frac{4T_0}{\beta} \right) }\\
    &\leq \exp{\left(\frac{1}{\beta}||E||_{2,2} ||\bm{\alpha}||_2 + \frac{4T_0}{\beta} \right)}\\
    &\leq \exp{\left(\frac{c \sqrt{n} + 4T_0}{\beta} \right)},
\end{align*}
where the second step is due to triangle inequality,
the third step is plugging in the definition of $E$ and bounding the second term based on the worst-case $X'_{pre}-X_{pre}$, which is all zeros with just one row with all $2$'s, and worst-case $y_{pre}$, which is all $1$'s). The fourth step is the submultiplicative property of operator norms, and the final step is due to the fact that $||E||_{2,2}=\bm{\lambda}_{|max|}(E)\leq c$ and that all elements of $\bm{\alpha} \in [-1,1]^n$ are bounded by $1$.
Then setting $\beta \geq \frac{c\sqrt{n}+4T_0}{\eps_0}$ ensures $\Gamma(\alpha) \leq e^{\eps_0}$.

If either of the above conditions on $\beta$ holds, then $\Gamma(\bm{\alpha}) \leq e^{\eps_0}$ as desired. Thus we can choose $\beta = \min\{ \frac{4 T_0 \sqrt{8+n}}{\eps_0}, \frac{c\sqrt{n}+4T_0}{\eps_0}\}$ that at least one will be satisfied. Taking the minimum rather than just one allows for a lower $\beta$ and hence lower noise magnitude, while still satisfying the privacy requirement.
\end{proof}

\subsection{Proof of Lemma \ref{lem.step4gauss}}\label{app.lemstep4gauss}

\stepfourgauss*

The proof of Lemma \ref{lem.step4gauss} follows a similar structure to Lemma 14 of \cite{KST12}. We include the full proof for completeness.
\begin{proof}
Let the noise term $\bm{b}$ be sampled from a multivariate Gaussian distribution $\mathcal{N}(0, \beta^2 I_n)$, and let $\D$ and $\D'$ be two arbitrary neighboring databases.
Let $h(\bm{\alpha}) = \bm{b}(\bm{\alpha};\D') - \bm{b}(\bm{\alpha};\D)$ Then, we can express $\Gamma(\bm{\alpha})$ as,
\begin{align}\label{eq.gammabound}
        \Gamma(\bm{\alpha}) &= \frac{\exp(-\frac{ ||\bm{b}(\bm{\alpha};\D)||_2^2 }{2\beta^2}) }{\exp(-\frac{ ||\bm{b}(\bm{\alpha};\D')||_2^2 }{2\beta^2}) }\notag\\
        &= \exp(\frac{1}{2\beta^2} ( ||\bm{b}(\bm{\alpha};\D')||_2^2  - ||\bm{b}(\bm{\alpha};\D)||_2^2  ) )\notag\\
        &= \exp(\frac{1}{2\beta^2} ( ||\bm{b}(\bm{\alpha};\D) + h(\bm{\alpha})||_2^2  - ||\bm{b}(\bm{\alpha};\D)||_2^2  ) )\notag\\
        &= \exp(\frac{1}{2\beta^2} ( 2\langle\bm{b}(\bm{\alpha};\D), h(\bm{\alpha})\rangle  + ||h(\bm{\alpha})||_2^2  ) ),
\end{align}
where the first step is from the distribution of noise $\bm{b}$, the final step is a binomial expansion applied to norms.

Note that,
\begin{align*}
        h(\bm{\alpha}) &= \bm{b}(\bm{\alpha};\D') - \bm{b}(\bm{\alpha};\D)\\
        &= T_0 (\nabla\cL(\bm{\alpha};\D) - \nabla\cL(\bm{\alpha};\D') )\\
        &= -T_0 \nabla g(\bm{\alpha}),
\end{align*}
where $g(\bm{\alpha}) = \cL(\bm{\alpha},\mathcal{D'})
- \cL(\bm{\alpha},\mathcal{D})$, as defined in Equation \eqref{eq.g}. By Lemma \ref{lem.goff}, we know that $||\nabla g(\bm{\alpha})||_2 \leq 4\sqrt{8+n}$, so also
\begin{equation}\label{eq.hbound}
||h(\bm{\alpha})||_2 \leq 4 T_0 \sqrt{8+n}. \end{equation}

Similarly, because $\bm{b}$ is sampled from a multivariate Gaussian distribution $\mathcal{N}(0, \beta I_n)$ and sum of Gaussian variables is also Gaussian, then,
\[
\langle \bm{b}(\bm{\alpha};\D), h(\bm{\alpha})\rangle \; \sim \; \mathcal{N}(0, \beta^2 || h(\bm{\alpha})||_2^2).
\]
Since the exact distribution is known, we use a Gaussian tail bound to find a \emph{well-behaving} set of $\bm{b}$.
\begin{lemma}[Chernoff bound for Gaussian \cite{wainwright2019high}]\label{lem.chernoffgauss}
Let $Z \sim \mathcal{N}(0, \sigma^2)$. Then, for all $t>\sigma$,
\[
P[Z\geq t] \leq \exp(-\frac{t^2}{2\sigma^2}).
\]
\end{lemma}

We instantiate Lemma \ref{lem.chernoffgauss} with $Z = \langle\bm{b}(\alpha;\D), h(\bm{\alpha})\rangle$ and $t = \beta || h(\bm{\alpha})||_2 \sqrt{2\log \frac{2}{\delta}}$. Note that $t>\sigma$ for any $\delta>1/2$. Then,
\[
\Pr\left[\langle\bm{b}(\alpha;\D), h(\bm{\alpha})\rangle \geq \beta || h(\bm{\alpha})||_2 \sqrt{2\log \frac{2}{\delta}} \right] \leq \frac{\delta}{2},
\]
which, by Equation \eqref{eq.hbound} implies that,
\begin{equation}\label{eq.goodbound}
   \Pr\left[\langle\bm{b}(\alpha;\D), h(\bm{\alpha})\rangle \geq \beta (4 T_0 \sqrt{8+n}) \sqrt{2\log \frac{2}{\delta}} \right] \leq \frac{\delta}{2}. 
\end{equation}

Define a set of values of $\bm{b}$, corresponding the the good event described by Equation \eqref{eq.goodbound}: $\textbf{GOOD} = \{ \bm{b} \; | \; \langle\bm{b}(\bm{\alpha};\D), h(\bm{\alpha})\rangle \leq \beta (4 T_0 \sqrt{8+n}) \sqrt{2\log \frac{2}{\delta}}\}$. By definition, $\Pr[\bm{b} \in \textbf{GOOD}] \geq 1 - \delta$. That is, with probability at least $1-\delta$, the noise vector $\bm{b}$ is in the well-behaving set \textbf{GOOD}. 

When $\bm{b}\in\textbf{GOOD}$, then we can complete the bound on $\Gamma(\bm{\alpha})$ from Equation \eqref{eq.gammabound}, combining the bound on $||h(\bm{\alpha})||_2^2$ from Equation \eqref{eq.hbound}:
\[
    \Gamma(\bm{\alpha}) = \exp(\frac{1}{2\beta^2}[ 2\langle\bm{b}(\bm{\alpha};\D), h(\bm{\alpha})\rangle  + ||h(\bm{\alpha})||_2^2  ] ) \leq \exp\left(\frac{1}{2\beta^2} [ 2\beta (4 T_0 \sqrt{8+n}) \sqrt{2\log \frac{2}{\delta}} + (4 T_0 \sqrt{8+n})^2 ] \right).
\]

Finally, the goal is to bound $\Gamma(\bm{\alpha}) \leq e^{\eps_0}$, in the case where $\bm{b} \in \textbf{GOOD}$. Solving the expression above for $\beta$ yields
\begin{align}\label{eq.Gamma.bound}
\beta &\geq \frac{1}{2} \left( \frac{(4 T_0 \sqrt{8+n}) \sqrt{2 \log \frac{2}{\delta}}}{\eps_0} + \sqrt{\frac{(4 T_0 \sqrt{8+n})^2 2 \log \frac{2}{\delta}}{\eps_0^2} + \frac{(4 T_0 \sqrt{8+n})^2}{\eps_0}} \right) \notag \\
&= \frac{1}{2}\left( \frac{4 T_0 \sqrt{8+n}}{\eps_0} \left(\sqrt{2 \log \frac{2}{\delta}} + \sqrt{2 \log \frac{2}{\delta} + \eps_0}\right)
\right)
\end{align}

Note that choosing
\[
\beta \geq \frac{(4 T_0 \sqrt{8+n}) \sqrt{2 \log \frac{2}{\delta} + \eps_0}}{\eps_0}
\]
satisfies the bound of Equation \eqref{eq.Gamma.bound}. 

Thus $\Gamma(\bm{\alpha})\leq e^{\eps_0}$, conditioned on $\bm{b} \in \textbf{GOOD}$, which occurs with probability at least $1-\delta$.
\end{proof}

\subsection{Proof of Lemma \ref{lem.fobjacc}}\label{app.lemfobjacc}

\fobjacc*

\begin{proof}
Recall the objective functions $J^{obj}$ and $J^{reg}$:
\[ J^{obj}(\f) = \cL(\f) + \frac{\lambda + \Delta}{2T_0} \|\f\|_2^2 + \frac{1}{T_0}\bm{b}^{\top}\f \quad \mbox{and } \quad J^{reg}(\f) = \cL(\f) + \frac{\lambda}{2T_0} \|\f\|_2^2, \]
with their respective minimizers $\f^{obj}$ and $\f^{reg}$. Define another objective function $J^{\#}$ and its minimizer $\f^{\#}$,
\[J^{\#}(\f) = \cL(\f) + \frac{\lambda + \Delta}{2T_0} \|\f\|_2^2 \]
which is a noise-free variant of $J^{obj}$.

We will express the difference between $\f^{reg}$ and $\f^{obj}$ using $\f^{\#}$ as an intermediate value:
\begin{equation}\label{eq.fregfobj}
   \|\f^{reg}-\f^{obj}\|_2 = \|\f^{reg} -\f^{\#} + \f^{\#} -\f^{obj}\|_2 
\leq \|\f^{reg} -\f^{\#}\|_2 + \|\f^{\#} -\f^{obj}\|_2. 
\end{equation}
We will bound these two terms separately, starting with $\|\f^{\#} -\f^{obj}\|_2$. It is known that $J^{obj}$ is $(\frac{\lambda + \Delta}{T_0})$-strongly convex, and that the gradient of of $J^{obj}$ evaluated at its minimizer $\f^{obj}$ is zero. Then by the definition of strong convexity, 
\begin{equation}\label{eq.fsharpfobj}
    \|\f^{\#}-\f^{obj}\|^2_2 \leq \left(J^{obj}(\f^{\#}) - J^{obj}(\f^{obj}) \right)\frac{2T_0}{\lambda+\Delta}.
\end{equation} 
We can proceed to bound the difference in the objective function $J^{obj}$ at these two points:
\begin{align*}
    J^{obj}(\f^{\#}) - J^{obj}(\f^{obj}) &= \left(J^{\#}(\f^{\#}) + \frac{1}{T_0}b^\top \f^{\#}\right) - \left(J^{\#}(\f^{obj}) + \frac{1}{T_0}b^\top \f^{obj} \right)\\
    &= \left(J^{\#}(\f^{\#}) - J^{\#}(\f^{obj})\right) + \left(\frac{1}{T_0}b^\top \f^{\#} - \frac{1}{T_0}b^\top \f^{obj}\right)\\
    &\leq 0 + \frac{1}{T_0} \|\bm{b}\|_2 \|\f^{\#}-\f^{obj}\|_2
\end{align*}
where the inequality is due to the fact that $J^{\#}(\f^{\#}) \leq J^{\#}(\f^{obj})$, since $\f^{\#}$ is the minimizer of $J^{\#}$. 

Plugging this into Equation \eqref{eq.fsharpfobj} gives
\[ \|\f^{\#}-\f^{obj}\|^2_2 \leq \frac{1}{T_0} \|\bm{b}\|_2 \|\f^{\#}-\f^{obj}\|_2\frac{2T_0}{\lambda+\Delta},\]
or equivalently,
\[ \|\f^{\#}-\f^{obj}\|_2 \leq \|\bm{b}\|_2 \frac{2}{\lambda+\Delta}.\]

To bound the first term of Equation \eqref{eq.fregfobj}, we observe that if $\Delta=0$, then $J^{\#}=J^{reg}$ and thus $\f^{\#}=\f^{reg}$, so $\|\f^{reg} -\f^{\#}\|_2 = 0$. Thus we only need to bound the distance when $\Delta \neq 0$.

We can write $\f^{reg}$ and $\f^{\#}$ using their closed-form expressions,
\[\f^{reg} = (X_{pre}X_{pre}^\top + \frac{\lambda}{2T_0} I)^{-1}X_{pre}\bm{y_{pre}} \quad \text{and} \quad \f^{\#} = (X_{pre}X_{pre}^\top + \frac{\lambda+\Delta}{2T_0} I)^{-1}X_{pre}\bm{y_{pre}}, \]
and use these to bound the difference:
\begin{align}\label{eq.obj.acc.1}
        \|\f^{reg} - \f^{\#}\|_2 &= \|\left( (X_{pre}X_{pre}^\top + \frac{\lambda}{2T_0} I)^{-1} - (X_{pre}X_{pre}^\top + \frac{\lambda+\Delta}{2T_0} I)^{-1} \right)X_{pre}\bm{y_{pre}}\|_2 \notag\\
        &\leq \left( (\|(X_{pre}X_{pre}^\top + \frac{\lambda}{2T_0} I)^{-1}\|_2 + \| ((X_{pre}X_{pre}^\top + \frac{\lambda+\Delta}{2T_0} I)^{-1})^{-1}\|_2  \right) \|X_{pre}\bm{y_{pre}}\|_2
\end{align}

The spectral norm of a general form $\|(XX^\top + \lambda I)^{-1}\|_2$ can be bounded by the inverse of minimum singular value of the matrix $XX^\top + \lambda I$, which is positive semi-definite and has minimum singular value at least $\lambda$:
\[ \|(XX^\top + \lambda I)^{-1}||_2 \leq \frac{1}{\bm{\sigma_{min}}(XX^\top + \lambda I)} \leq \frac{1}{\lambda}. \]

Using this fact, we can further bound Equation \eqref{eq.obj.acc.1} as,
\begin{align*}
    \|\f^{reg} - \f^{\#}\|_2 & \leq \left( \frac{2T_0}{\lambda} + \frac{2T_0}{\lambda+\Delta} \right) \|X_{pre}\bm{y_{pre}}\|_2 \\
    &\leq 2T_0 \left( \frac{1}{\lambda} + \frac{1}{\lambda+\Delta} \right) ||X_{pre}||_F ||\bm{y_{pre}}||_2\\
    &\leq 2T_0 \left( \frac{1}{\lambda} + \frac{1}{\lambda+\Delta} \right) \sqrt{nT_0} \sqrt{T_0} \\
    &= \left( \frac{1}{\lambda} + \frac{1}{\lambda+\Delta} \right) 2T_0^2 \sqrt{n}.
\end{align*}

Finally, we combine Equation \eqref{eq.fregfobj} with bounds on both terms to yield: 
\begin{align*}
        \E[\|\f^{reg}-\f^{obj}\|_2 ] &\leq \E[\|\f^{\#} -\f^{obj}\|_2] + \mathds{1}_{\Delta \neq 0} \E[\|\f^{reg} - \f^{\#}||_2] \\
        &\leq \frac{2}{\lambda+\Delta}\E[\|\bm{b}\|_2] + \mathds{1}_{\Delta \neq 0} \left( \frac{1}{\lambda} + \frac{1}{\lambda+\Delta} \right) 2T_0^2 \sqrt{n}.
\end{align*}

\end{proof}

\end{document}